%% file: main.tex
\documentclass[runningheads]{llncs}
\usepackage[T1]{fontenc}
\usepackage{graphicx}
\usepackage{booktabs}
\usepackage[misc]{ifsym}
\usepackage{symbolDef}
\usepackage{xcolor}
\usepackage{pifont}
\newcommand{\Cestt}{\mthC_{\sct}}
\usepackage{multirow}
\usepackage{cleveref}
\usepackage{subcaption}
\usepackage{amsmath}
\usepackage{listings}
\usepackage{enumitem}
\usepackage{url}
\usepackage{xr}
\newcommand{\corr}{(\Letter)}
\newtheorem{Assumption}{As.}
\usepackage{mwe}

\begin{document}

\title{Spatiotemporal Covariance Neural Networks}


\author{Andrea Cavallo \corr \and
Mohammad Sabbaqi \and
Elvin Isufi}

\authorrunning{A. Cavallo et al.}

\institute{Delft University of Technology, Delft, Netherlands \\ \email{\{a.cavallo,m.sabbaqi,e.isufi-1\}@tudelft.nl}
}

\tocauthor{Andrea~Cavallo, Mohammad~Sabbaqi, Elvin~Isufi}

\toctitle{Spatiotemporal Covariance Neural Networks}

\maketitle              

\begin{abstract}

Modeling spatiotemporal interactions in multivariate time series is key to their effective processing, but challenging because of their irregular and often unknown structure. Statistical properties of the data provide useful biases to model interdependencies and are leveraged by correlation and covariance-based networks as well as by processing pipelines relying on principal component analysis (PCA). 
However, PCA and its temporal extensions suffer instabilities in the covariance eigenvectors when the corresponding eigenvalues are close to each other, making their application to dynamic and streaming data settings challenging.  
To address these issues, we exploit the analogy between PCA and graph convolutional filters to introduce the SpatioTemporal coVariance Neural Network (STVNN), a relational learning model that operates on the sample covariance matrix of the time series and leverages joint spatiotemporal convolutions to model the data. To account for the streaming and non-stationary setting, we consider an online update of the parameters and sample covariance matrix. We prove the STVNN is stable to the uncertainties introduced by these online estimations, thus improving over temporal PCA-based methods.  
Experimental results corroborate our theoretical findings and show that STVNN is competitive for multivariate time series processing, it adapts to changes in the data distribution, and it is orders of magnitude more stable than online temporal PCA.

\end{abstract}

\section{Introduction}

Learning representations from multivariate time series is inherently challenging because of their spatiotemporal coupling, but it is relevant for a wide range of applications including multiple-location weather measurements~\cite{sanhudo2021mtsweather}, state-evolution in infrastructure networks such as water~\cite{kerimov2023assessing}, power~\cite{habib2024power} or transportation~\cite{Jiang_2022gnntraffic}, and brain activity in neuroimaging~\cite{pourahmadi2016brain}. When the data structure is irregular and unknown, correlation or covariance-based networks~\cite[Chapter 7.3.1]{kolaczyk2009statistical},\cite{bessadok2022graph,brooks2020timeseries,cardoso2020algorithms} are used as a proxy to capture their inter-dependencies, with the popular example of graphical lasso~\cite{friedman2007lasso}, among others. Learning representations with these networks as inductive biases is tied to dimensionality reduction techniques such as principal component analysis (PCA), which transforms the input to maximize the variance among data points~\cite{cardot2018online,Jolliffe2002pca}.

However, PCA on temporal data has three fundamental challenges. First, it cannot capture temporal dependencies among data points and temporal distribution shifts. Second, it works with batches of data, which are not available in a streaming setting. Third, it suffers instabilities in the principal components corresponding to eigenvalues that are close to each other~\cite{Jolliffe2016PrincipalCA}. To overcome the first two challenges, online and temporal variants of PCA have been proposed~\cite{cardot2018online,Jolliffe2002pca}. A simple instance of temporal PCA is in~\cite[Chapter 12.2]{Jolliffe2002pca}, which concatenates previous temporal samples and projects them on the eigenspace of their covariance matrix to capture spatiotemporal correlations. This requires working with a bigger spatiotemporal covariance matrix that affects scalability. Online PCA algorithms, instead, observe data samples one at a time and update the covariance matrix estimate or its spectral components as new samples appear~\cite{cardot2018online}.
However, the online update introduces additional uncertainty in the downstream PCA-based model, whose parameters are optimized on the limited portion of observed data. {To overcome these issues, we exploit the link between PCA and graph filtering~\cite{marques2017stationary,rui2016dimensionality,sihag2022covariance} to propose a SpatioTemporal coVariance Neural Network (STVNN) that builds on graph neural network principles and uses the covariance matrix as a graph representation matrix. This connection has been recently exploited in~\cite{sihag2022covariance} for static tabular data, showing enhanced stability compared to conventional PCA. However, the model in~\cite{sihag2022covariance} does not consider temporal and streaming data, and trivial extensions such as using temporal data as node features fail to effectively leverage the complex spatiotemporal interactions, ultimately affecting performance and deteriorating stability as we shall elaborate in Sec.~\ref{sec_stab}.
Motivated by this observation, we rely on two-dimensional spatiotemporal convolutional principles and on the covariance matrix to process the principal components of the time series.

\smallskip
\noindent\textbf{Contributions.} We investigate spatiotemporal covariance-based graph convolutional networks for multivariate time series to learn on-the-fly representations that are suitable for a streaming setting, adapt to data distribution shifts, and are robust to estimation errors in both the parameter and covariance matrix update. Our specific contribution is threefold.

\begin{enumerate}[label=(\textbf{C\arabic*)}, start = 1]
    \item \textbf{Principled architecture.} We define STVNN, a temporal graph neural network for multivariate time series that employs the sample covariance matrix as graph structure. STVNN is rooted in the convolution principle and draws analogies with temporal PCA. Differently from downstream models building on PCA, STVNN learns representations directly by acting on the principal components in each layer. We develop an online learning update to account for streaming data and distribution shifts.
    \item \textbf{Stability analysis.}  We prove that the convolution-based design of STVNN provides stability to uncertainties in the online estimation of covariance matrix and model parameters, even in the presence of close covariance eigenvalues, at the cost of discriminability. Moreover, the size of temporal information is significantly less harmful for stability compared to alternative models. 
    
    \item \textbf{Empirical evidence.} 
    We corroborate our theoretical findings with experiments on real and synthetic datasets. Our results show that STVNN is substantially more robust, adapts better to distribution shifts, and achieves a better multi-step forecasting performance than alternatives relying on PCA or covariance neural networks that ignore a joint spatiotemporal processing. Our code is available at \url{https://github.com/andrea-cavallo-98/STVNN}.
\end{enumerate}

\section{Related Works}

This work stands at the intersection of spatiotemporal modeling, streaming multivariate time series, principal component analysis, and graph neural networks (GNNs). Here, we position it w.r.t. prior art.

\smallskip
\noindent\textbf{PCA and graphs.}
The analogy between PCA and graph filters has been studied in the field of graph signal processing.
The works in~\cite{marques2017stationary,perraudin2017stationary} discuss stationarity for graph signals, which graph Fourier transform links to PCA. The connection between PCA and graph filters is exploited also in~\cite{rui2016dimensionality,shahid2016fast} for dimensionality reduction of graph data and designing PCA on graphs, respectively. To improve upon the PCA instabilities, the work in~\cite{sihag2022covariance} uses the covariance matrix as a graph shift operator to develop a neural network architecture which is more expressive and stable than PCA itself. This approach is further studied in~\cite{sihag2023transferablility} and used for brain data application~\cite{sihag2024explainable}. Yet, all these works consider static data.

\smallskip
\noindent\textbf{Online and temporal PCA.}
\label{sec:tpca}
Online PCA deals with streaming data by projecting them onto the estimated principal directions, which are then updated.
Different categories of online PCA approaches exist such as perturbation methods, incremental PCA, and stochastic approximation methods~\cite{brand2002incremental,cardot2018online}.
Instead, temporal PCA models interactions among data across time and space~\cite{Jolliffe2002pca}. One of the most common ideas is to consider, at each time instant, a previously observed batch of time samples to account for the dynamic evolution of the data. Then, temporal PCA estimates a spatiotemporal covariance matrix accounting for the correlation among all space-time locations. This leads to prohibitive costs and limited temporal memory. To overcome these limitations, we propose an alternative based on spatiotemporal graph convolution principles using the covariance matrix.

\smallskip
\noindent\textbf{Spatiotemporal GNNs.}
GNNs have become a popular tool for processing time series because of their ability to exploit known or latent relations when a consistent structure is available.
However, in some cases, these relations either change through time or a better estimation of them is at hand by observing new data.
To tackle the changes in the structure, \cite{wu2020connecting} connects time series with an attention mechanism, whereas~\cite{jin2022multivariate} proposes an ODE-based approach to learn an adjacency matrix among time series.
In other works, either conventional Fourier transform~\cite{yi2023frequencydomain} or graph Fourier transform~\cite{cao2020spectral,yi2023fouriergnn} is used.
To account for the incoming data, online learning methods are employed to adapt the predictors and the structure.
The work in~\cite{saadallah2023online} implements shift detectors to trigger model updates, while~\cite{pham2022learning} keeps a memory of relevant previous information and~\cite{guo2016robust} detects and excludes outliers from the updates to avoid catastrophic forgetting. 
Here, we study the effectiveness of the covariance matrix as an inductive bias to model data interdependencies, we exploit a principled online update of it to handle streaming data, and we characterize the impact of the estimation error.

\section{Problem Formulation}
\label{sec:problem_formulation}

We are interested in processing multivariate time series $\vcx_t\in\mathbb{R}^N$ comprising $N$ variables evolving in vectors $\ldots\vcx_{-1}, \vcx_0, \vcx_1,\ldots$ by taking into account their mean vector $\vcmu$ and covariance matrix $\mtC$, which under stationarity are defined as
\begin{equation}
    \vcmu=\mathbb{E}[\vcx_t]~~\textnormal{and}~~\mtC=\mathbb{E}[(\vcx_t-\vcmu)(\vcx_{t}-\vcmu)^\Tr].
\end{equation}
More specifically, we want to learn a function $\fnPhi(\vcx_{T:t}, \mtC; \vch)$ that takes as input a temporal memory of size $T$, $\vcx_{T:t} = \{\vcx_{t-T+1}, \ldots, \vcx_t\}$, and the covariance matrix $\mtC$ and maps them to a target output $\vcy$, which may be a future instance at horizon $\tau > 0$, i.e., $\vcy = \vcx_{t+\tau}$, or a class label. Here, $\vch$ comprises the parameters of function $\fnPhi(\cdot)$. The rationale to account for $\mtC$ is to exploit it as an inductive bias so as to reduce the parameters and computational complexity of the model in contrast to a model that learns directly from the data (e.g., an LSTM network), as well as develop a principled solution rooted in PCA.

However, in practice, we do not have access to the true covariance matrix $\mtC$ nor to a batch of data to accurately estimate it. Instead, we observe the evolution of the time series in a streaming fashion and update recursively the estimates of the mean $\vchmu_{t}$ and covariance matrix $\Cestt$ as
\begin{equation}
\label{eq:update_stationary}
\begin{gathered}
        \vchmu_{\sct+1} = \alpha_t\vchmu_{\sct} + \beta_t\vcx_{\sct+1} \\
    \mthC_{\sct+1} = \xi_t\mthC_{\sct} + \zeta_t(\vcx_{\sct+1}-\vchmu_{\sct})(\vcx_{\sct+1}-\vchmu_{\sct})^\Tr
\end{gathered}
\end{equation}
where $\alpha_t, \beta_t, \xi_t, \zeta_t \in [0,1]$ are scalars. The update in \eqref{eq:update_stationary} accounts for both the stationary setting -- $\alpha_t=t/(t+1)$, $\beta_t=\zeta_t=1/(t+1)$, and $\xi_t=(t-1)/t$ -- as well as the non-stationary setting -- $\alpha_t = \xi_t = 1-\gamma$ and $\beta_t = \zeta_t = \gamma$ with scalar $\gamma\in[0,1]$ regulating the contributions of the more recent and past data. 
Consequently, we also want to update the model parameters $\vch_{t+1}$ on-the-fly based on the current estimate $\vch_{t}$ and on the loss at time $t$ $\mathcal{L}(\fnPhi(\vcx_{T:t},\Cestt; \vch_t), \vcy)$.

\smallskip
\noindent\textbf{Relation to PCA and graph filtering.} The classical way to learn representations from the time series by accounting for their covariance matrix $\Cestt$ is to extract their principal components and use them as representative features for a downstream model. These features are achieved by first taking the eigendecomposition $\Cestt = \mthV\mathbf{\hat{\Lambda}}\mthV^\Tr$ with eigenvectors $\mthV$ and diagonal matrix of eigenvalues $\mathbf{\hat{\Lambda}}$, and then projecting vector $\vcx_t$ onto the eigenspace as $\vctx_t:=\mthV^\Tr\vcx_t$. However, this approach is sensitive to the statistical uncertainty that occurs when the number of streaming data points is limited or when the principal components correspond to two close eigenvalues~\cite{jolliffe1989illdefined,Jolliffe2016PrincipalCA}. Our objective is to build on the PCA principle by relying on the eigenspace of the covariance matrix but improve its robustness when dealing with streaming data. This will be achieved by relying on graph filtering principles~\cite{isufi2024gsp} as we elaborate next.

We can consider $N$-dimensional series $\vcx_t$ as time-varying signals associated to an $N$-node undirected graph. Thus, the covariance matrix $\mthC_{\sct}$ in \eqref{eq:update_stationary} becomes equivalent to its time-varying weighted adjacency matrix, which in a stationary setting converges to the true covariance matrix $\mtC$. Consequently, the PCA at time $t$ is equivalent to the graph Fourier transform w.r.t. the adjacency matrix $\mthC_{\sct}$~\cite{rui2016dimensionality,sihag2022covariance}. In turn, preserving a few of the signal principal components is equivalent to a spectral graph filtering operation on $\vctx_t$, which has been analyzed in graph signal processing under the lens of \emph{graph} stationarity~\cite{marques2017stationary}. This approach is inapplicable due to the high computational cost associated with the eigendecomposition and it limits the scope to a transductive setting. To overcome this, we rely on the graph convolutional principle that operates directly in node domain and can achieve a higher robustness than PCA as evidenced in~\cite{sihag2022covariance} for tabular data. Specifically, an order $K$ graph convolutional filter operating on $\vcx_t$ gives the output $\vcz_t = \sum_{k=0}^Kh_k\mthC_{\sct}^k\vcx_t = \mtH(\mthC_{\sct})\vcx_t$ where $h_k$ are learnable parameters. By projecting both the filter input and output onto the graph spectral domain, a.k.a. computing the graph Fourier transform, we obtain
\begin{align}
\vctz = \mthV^\Tr\vcz = \mthV^\Tr\sum_{k=0}^Kh_k[\mthV\mathbf{\hat{\Lambda}}\mthV^\Tr]^k\vcx = \sum_{k=0}^Kh_k\mathbf{\hat{\Lambda}}^k\mthV^\Tr\vcx.
\end{align}
For the $i$-th entry, this implies
$[\vctz]_{i} =  \sum_{k=0}^Kh_k\hat{\lambda}_i^k[\vctx]_{i} = h(\hat{\lambda}_i)[\vctx]_{i}$,
i.e., the frequency response of the graph convolutional filter is a polynomial $h(\lambda)$ in the eigenvalues $\hat{\lambda}_i$. The work in~\cite{sihag2022covariance} showed that there exists a combination of parameters $h_k$ such that graph convolutional filters on covariance matrix perform PCA for tabular data. Yet, this is inapplicable to our setting as it ignores the temporal dependencies of the data which are key for spatiotemporal processing.

\begin{remark} 
\label{remark1}
The discussion so far considers only the so-called lag-zero covariance matrix and not the covariances between variables at different times, i.e., $\mtC_\tau=\mathbb{E}[(\vcx_t-\vcmu)(\vcx_{t-\tau}-\vcmu)^\Tr]$. These are relevant to perform temporal PCA~\cite{Jolliffe2002pca} as they allow capturing cross-dependencies between the time series. Yet, this implies working with the eigendecomposition of a bigger matrix (see Appendix~\ref{app:tpca}), which affects scalability. To overcome the latter, we shall rely on a two dimensional filter working only on $\mtC$. Our choice could also be seen as working with temporal independence between elements $\vcx_t,\vcx_{t^\prime}$ when $t^\prime \neq t$ but we shall see this is not the case as our model has an explicit temporal memory [cf. Def.~\ref{def.TVFilter}]. This improves scalability as we will work only with a single covariance matrix rather than with $T$ different covariance matrices and it allows us to draw parallels with the online PCA as well as with the work in \cite{sihag2022covariance}. 
\qed \end{remark}

\section{Spatiotemporal Covariance Neural Networks}

\begin{figure}[t]
     \centering
     \includegraphics[width=\textwidth]{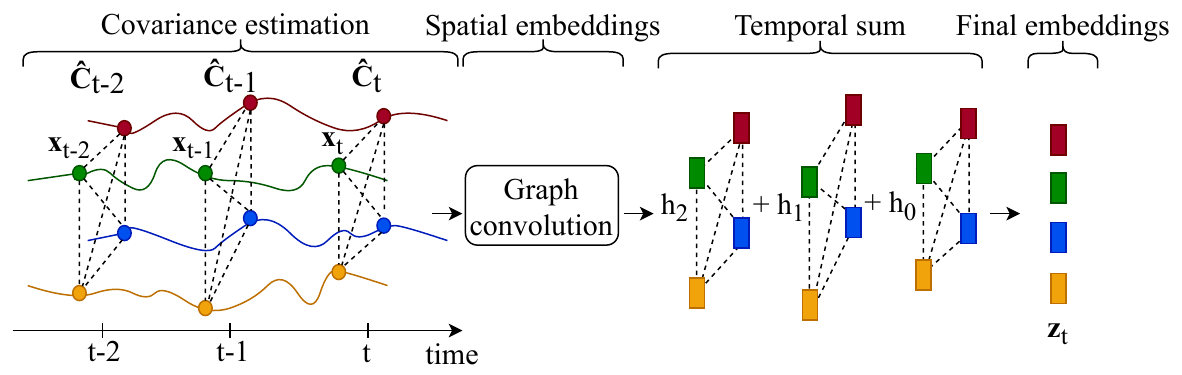}
     \caption{Spatiotemporal covariance filter pipeline. We observe a new time series sample $\vcx_t$ and update the covariance estimate $\Cestt$, which the STVF uses as a weighted graph adjacency matrix to model spatial interactions. Then, we sum the embeddings of the last $T$ samples to model temporal interactions.}
     \label{fig:tvnn}
\end{figure}

Key to developing the spatiotemporal covariance neural network is the SpatioTemporal coVariance Filter (STVF).

\begin{definition}[Spatiotemporal covariance filter (STVF)]\label{def.TVFilter} The input-output relation of a STVF of temporal memory $T$ and spatial memory $K$ is defined as
\begin{equation}
\label{eq:tv_filter}
    \vcz_{\sct} :=\mtH(\Cestt, \vch_t, \vcx_{T:t}) = \sum_{\sct'=0}^{\scT-1} \sum_{\sck=0}^\scK \sch_{\sck\sct'}\Cestt^\sck\vcx_{\sct-\sct'}
\end{equation}
where $\sch_{\sck\sct'}$ are the filter parameters. 
\end{definition}

The STVF resembles a joint shift and sum operation in both time and space, which is the building principle of the convolutional filters \cite{isufi2024gsp}. The spatiotemporal memory of the filter is controlled by the orders $(T,K)$ ultimately accounting for neighboring information up to $T$ temporal lags $\vcx_{t-T+1}, \ldots, \vcx_t$ and up to $K$ hops away in the covariance graph $\Cestt$. We illustrate the data processing pipeline of the STVF in \Cref{fig:tvnn}. 

The operation in \eqref{eq:tv_filter} can also be seen as a filter bank of $T$ covariance filters $\mtH_{\sct'}(\Cestt, \vch_t):=\sum_{\sck=0}^\scK \sch_{\sck\sct'}\Cestt^\sck$ each acting on the $t^\prime$ lag $\vcx_{t-t^\prime}$, i.e., $\vcz_{\sct} = \sum_{\sct'=0}^{\scT-1}\mtH_{\sct'}(\Cestt, \vch_t)\vcx_{\sct-\sct'}$. 
Therefore, we can compute the graph Fourier transform of the STVF separately for each filter at different lags, i.e., 
\begin{align}
\vctz_{\sct} = \mthV^\Tr\vcz_t = \sum_{\sct'=0}^{\scT-1} \sum_{\sck=0}^\scK \sch_{\sck\sct'}\mathbf{\hat{\Lambda}}^\sck\mthV^\Tr\vcx_{\sct-\sct'}.
\end{align}
Analogously to standard graph convolutional filters, this implies that, for each component $i$, the frequency response of the filter is a sum of polynomials in the covariance eigenvalues, i.e., 
\begin{equation}
    [\vctz_{\sct}]_i = \sum_{\sct'=0}^{\scT-1} \sum_{\sck=0}^\scK \sch_{\sck\sct'}\hat{\lambda}_i^\sck[\vchx_{\sct-\sct'}]_i = \sum_{\sct'=0}^{\scT-1} h_{t'}(\hat{\lambda}_i)[\vchx_{\sct-\sct'}]_i.
\end{equation} 
Based on the result from~\cite[Theorem 1]{sihag2022covariance}, there exists a combination of coefficients $h_{kt'}$ such that each filter at different lags performs PCA on the lag-zero covariance matrix.
This is different from the temporal PCA defined in~\Cref{sec:tpca}, which also considers covariances at different lags (see~\Cref{app:tpca} for details), whereas the STVF models temporal interdependencies of samples through its temporal memory.

Given the STVF, we now introduce the STVNN architecture.

\begin{definition}[SpatioTemporal coVariance Neural Networks (STVNN)]\label{def.TVNN} A STVNN is a layered architecture where each layer $l = 1, \ldots, L$ comprises a STVF [cf. Def.\ref{def.TVFilter}] nested into a pointwise nonlinearity $\sigma(\cdot)$, i.e.,
\begin{equation}
    \vcz_t^l = \sigma\left(\mtH^l(\Cestt, \vch_t, \vcz_{T:t}^{l-1})\right) =
    \sigma\left( \sum_{\sct'=0}^{\scT-1} \sum_{\sck=0}^\scK \sch_{\sck\sct'}^l\Cestt^\sck\vcz_{\sct-\sct'}^{l-1} 
    \right),~~\textnormal{for}~~l = 1, \ldots, L
\end{equation}
with input $\vcz_{\sct-\sct'}^{0} := \vcx_{\sct-\sct'}$.
\end{definition}

This enables learning joint, nonlinear, and layered representations from spatiotemporal time series by leveraging the covariance matrix $\Cestt$ as inductive bias. The output of the last layer $\vcz_t^L$ contains the final representations and constitutes the output of the STVNN, i.e., $\vcz_t^L := \fnPhi(\vcx_{T:t}, \Cestt; \vch)$. The parameters comprise those of all the filters in all layers $\vch = \{h_{kt^\prime}^l\}_{kt^\prime l}$ and are of order $\mathcal{O}((K+1)TL)$. To further improve the representational capacity of the STVNN, parallel filter banks are aggregated in each layer. That is, at layer $l$ there are $F_{\text{in}}$ input features $[\vcz_{T:t}^{l-1}]_1, \ldots, [\vcz_{T:t}^{l-1}]_{F_{\text{in}}}$ and $F_{\text{out}}$ output features $[\vcz_{T:t}^{l}]_1, \ldots, [\vcz_{T:t}^{l}]_{F_{\text{out}}}$ coupled by a filterbank of $F_{\text{in}} \times F_{\text{out}}$ filters as
\begin{equation}
\label{eq:stvnn_filterbank}
    [\vcz_t^l]_f = \sigma\left(\sum_{g = 1}^{F_{\text{in}}}\mtH^l(\Cestt, \vch_t, [\vcz_{T:t}^{l-1}]_g)\right)~f = 1, \ldots, F_{\text{out}}, l = 1, \ldots, L.
\end{equation}

\smallskip
\noindent\textbf{Online learning.} The STVNN output represents the spatiotemporal covariance-based embeddings of the time series that can be either used directly for a downstream task or processed further by a readout layer. To deploy the STVNN in a data streaming setting and make it adaptable to distribution shifts, we resort to online machine learning principles to update the filter parameters. Specifically, with a slight abuse of notation, let $\vch_t$ denote the parameter vector of the model (either filter or neural network) at time $t$. Then, as a new time sample $\vcx_t$ becomes available at time $t$, we update the parameters as
\begin{equation}\label{eq.onlinelearn}
    \vch_{\sct+1} = \vch_{\sct} - \sceta\nabla_\sct\mathcal{L}(\fnPhi(\vcx_{T:t},\Cestt; \vch_t)),
\end{equation}
where $\sceta > 0$ is the learning rate.

\smallskip
\noindent\textbf{Computational complexity.} For a STVNN layer in \eqref{eq:stvnn_filterbank}, the computational complexity is of the order $\mathcal{O}(N^2TKF_\text{in}F_\text{out})$. Although the term $N^2$ makes the model inefficient on large datasets, frequently sparse estimates of large-dimensional covariance matrixes are computed~\cite{bien2011sparse}. This reduces the complexity to $\mathcal{O}(|E|TKF_\text{in}F_\text{out})$, where $|E|$ is the number of non-zero correlations. Moreover, note that the complexity scales linearly with the temporal size $T$, whereas the complexity of temporal PCA projection scales as $\mathcal{O}(N^2T^2)$, which makes it significantly less suitable for large datasets.

\section{Stability Analysis}\label{sec_stab}

In a data streaming setting, the STVNN operates on the sample covariance matrix $\Cestt$ estimated online from finite data up to time $t$. Consequently, $\Cestt$ is a perturbed version of the underlying covariance matrix $\mtC$, which in turn induces a perturbation in the embeddings w.r.t. an STVNN trained with the true matrix $\mtC$. While for PCA-based statistical learning models this is a notorious challenge, we show in the sequel that the STVNN is robust to these finite-data effects. Likewise, the model parameters $\vch_t$ are also updated online, which calls for a suboptimality regret-like analysis w.r.t. the optimal parameters in hindsight $\vch^*$.

\subsection{Stability of spatiotemporal covariance filter}
\label{sec:stability_filter}

To study the stability of STVF, we characterize the distance between the online finite-data trained filter output $\mtH(\Cestt, \vch_t, \vcx_{T:t})$ and the output of the filter trained with the underlying covariance matrix and all the data $\mtH(\mtC, \vch^*, \vcx_{T:t})$. We then make the following standard assumptions and claim our main result.

\begin{Assumption}\label{as_lipschitz} The frequency response $h_{t'}(\lambda)$ of an STVF is Lipschitz. That is, there exists a constant $P > 0$ such that 
\begin{equation}
    |h_{t'}(\lambda_i)-h_{t'}(\lambda_j)| \leq  P|\lambda_i - \lambda_j|, \quad t'=0,\dots,T-1.
\end{equation}
\end{Assumption}

\begin{Assumption}\cite[Theorem 5.6.1]{vershynin2018high}
\label{as_norm}
    For a multivariate time series $\vcx_t$ with covariance matrix $\mtC$ the following holds:
    \begin{equation}
        \mathbb{P} \left( \|\vcx_t\|\leq G\sqrt{\mathbb{E}[\|\vcx_t\|^2]}\right) \ge 1 - \delta
    \end{equation}
where $G \ge 1$ is a constant, $\delta \approx 0$, and $\|\cdot\|$ is the $\ell_2-$norm.   
\end{Assumption}

\begin{Assumption}\cite[Theorem 4.1]{loukas2017howclose}
\label{as_eig_diff}
    Given the frequency response $h_{t'}(\lambda)$ of an STVF, the eigenvalues $\{\lambda_i\}_{i=0}^{N-1}$ and $\{\schlambda_i\}_{i=0}^{N-1}$ of the true and sample covariance matrix, respectively, satisfy
    \begin{align}
        \textnormal{sign}(\lambda_{i} - \lambda_{j})2\schlambda_{i} > \textnormal{sign}(\lambda_{i} - \lambda_{j})(\lambda_{i} + \lambda_{j})
    \end{align}
    for each pair of distinct eigenvalues $(\lambda_i,\lambda_j)$.
\end{Assumption}

As.~\ref{as_lipschitz} limits the discriminability of the filters, as their frequency response cannot change with a slope larger than $P$. As.~\ref{as_norm} refers to the variance of the data distribution, as $G$ is higher for data with higher variance. As.~\ref{as_eig_diff} relates to the approximation error of the sample covariance eigenvalues w.r.t. the true ones and holds for each eigenvalue pair $(\lambda_i,\lambda_j)$ with probability at least $1-2k_{i}^2/(N|\lambda_{i} - \lambda_{j}|)$, where $k_{i}=\left( \mathbb{E}[\|\vcx_t\vcx_t^\Tr\vcv_{i}\|^2]-\lambda_{i}^2 \right)^{1/2}$ is a term related to the kurtosis of the data distribution~\cite[Corollary 4.2]{loukas2017howclose}.

\begin{theorem}
\label{th:stability_filter}
Consider a multivariate time series $\vcx_t\in\mathbb{R}^N$ with underlying covariance matrix $\mtC$ and let the sample covariance matrix estimated from $t$ samples be $\Cestt$. Additionally let the instances satisfy w.l.o.g. $\|\vcx_t\|\leq 1$, let As.~\ref{as_lipschitz}-\ref{as_eig_diff} hold, and let the learning rate $\eta>0$ be small enough to guarantee convergence. Denote also the filter parameters optimized over the complete dataset by $\vch^*$ and those optimized online over $t$ samples using the online update in \eqref{eq.onlinelearn} by $\vch_t$. Then, the following holds with probability at least $(1-e^{-\epsilon})(1-2e^{-u})$:
\begin{equation}
\label{eq:tvnn_filter_stab}
\begin{gathered}
    \left\|\mtH(\Cestt, \vch_t, \vcx_{T:t}) - \mtH(\mtC, \vch^*, \vcx_{T:t})\right\| \leq \\
     \underbrace{\frac{1}{\sqrt{t}}PTN\left(k_\textnormal{max}e^{\epsilon/2} + QG\|\mtC\|\sqrt{\log{N}+u}\right)}_\textnormal{covariance uncertainty}
     + \underbrace{\frac{\|\vch^*\|^2}{2\eta t}}_\textnormal{parameter suboptimality} + \mathcal{O}\left(\frac{1}{t}\right)
\end{gathered}
\end{equation}
where $Q$ is an absolute constant, $k_{\textnormal{max}} = \max_jk_j$, and $k_j=\left( \mathbb{E}[\|\vcx_t\vcx_t^\Tr\vcv_j\|^2]-\lambda_j^2 \right)^{1/2}$ is related to the kurtosis of the data distribution. Here, $\epsilon, u >0$ can be arbitrarily large and $\|\cdot\|$ denotes the $\ell_2-$norm for vectors and the spectral norm for matrices. 
\end{theorem}

\begin{proof}
See Appendix~\ref{app:proof_stability_filter}.    
\end{proof}

The result in \eqref{eq:tvnn_filter_stab} highlights the role of the online update uncertainties in both the covariance matrix ($\|\mtH(\Cestt, \vch^*, \vcx_t) - \mtH(\mtC,\vch^*, \vcx_t)\|$) and the filter parameters ($\|\mtH(\Cestt,\vch_t, \vcx_t) - \mtH(\Cestt, \vch^*, \vcx_t)\|$). We make the following observations. 

\smallskip
\emph{Number of time samples.} The bound decreases with a rate $\mathcal{O}\left(1/\sqrt{t}\right)$ that is associated to the covariance matrix uncertainty. The bound also shows that the filter update plays a minor role on stability.\footnote{Note that in \Cref{eq:tvnn_filter_stab} we kept explicitly the parameter suboptimality term outside the notation $\mathcal{O}(1/t)$ to highlight the role of filter updates in the stability. The other terms $\mathcal{O}(1/t)$ include perturbations due to the covariance uncertainty.} 

\smallskip
\emph{Temporal window size.} A larger size $T$ of the STVF temporal memory in~\eqref{eq:tv_filter} implies a lower stability. This is because the uncertainties in the covariance matrix updates are propagated in $T$ filters to build the output. While a larger temporal memory may improve the filter discriminability, it may be damaging in non-stationary settings as the output will depend on past irrelevant history. In the stationary case, result \eqref{eq:tvnn_filter_stab} shows that it may not be useful when the covariance matrix is estimated from a few data points.

    \smallskip
\emph{Data distribution.} The term $k_j$ is related to the kurtosis of the data distribution at $\tau = 0$ along the $\vcv_i$ direction. More in detail, distributions with low kurtosis (low $k_i$) tend to have fast decaying tails, which makes the estimation of $\vcv_i$ easier and increases the stability~\cite{loukas2017howclose}. We shall corroborate this with numerical experiments in \Cref{sec:numerical_stability}.

\smallskip    
\emph{Optimal coefficients.} For a forecasting setting via the filter, we can characterize in closed form the role of the optimal parameters $\vch^*$ and their impact on the bound. Suppose we observe in batch $M$ temporal observations of the time series $\mtX\in\real^{\scN\times\scM}$ and estimate the covariance matrix $\mthC_{\scM}$. Let also the predicted targets be grouped in $\mtY\in\real^{\scN\times\scM}$. Then, fitting the data $\mtX$ into the target $\mtY$ via the filter \eqref{eq:tv_filter} implies solving $\min_\vch \|\mtA\vch-\vchy\|^2$, where vector $\vchy\in\real^{\scM\scN}$ concatenates the columns of $\mtY$, vector $\vch = \{\sch_{kt'}\}\in\real^{\scT(\scK+1)}$ collects all filter parameters, and matrix $\mtA\in\real^{\scM\scN\times\scT(\scK+1)}$ collects shifted version of the input $\mthC_{\scM}^\sck\vcx_{\sct'-\sct}~\forall t^\prime, t, k$. For a sufficiently large training set, this is an overdetermined system of equations with solution $\vch^*=(\mtA^\Tr\mtA)^{-1}\mtA^\Tr\vchy=\mtA^+\vchy$. If matrix $\mtA$ has a bad condition number, which relates to how the input is shifted over the estimated covariance, then the energy of $\vch^*$ explodes contributing to a looser bound. We may redeem this by solving a regularized least-squares problem. 

\subsection{Stability of STVNN}

We now extend the stability analysis w.r.t. the covariance uncertainty to the STVNN composed of $L$ layers and $F$ filterbanks per layer.
\begin{theorem}
Consider a true covariance matrix $\mtC$, a sample covariance matrix $\Cestt$ and a bank of STVF with frequency response terms $|h_{t'}(\lambda)|\leq 1$ and non-linearity $\sigma(\cdot)$ such that $|\sigma(a)-\sigma(b)|\leq |a-b|$. If the filters satisfy $\|\mtH(\Cestt, \vch^*, \vcx_t) - \mtH(\mtC,\vch^*, \vcx_t)\| \leq \beta_t$ for a generic $\beta_t$, then the STVNN satisfies
\begin{equation}
    \label{eq:tvnn_stability}
    \| \fnPhi(\Cestt,\vch^*, \vcx_{T:t}) - \fnPhi(\mtC,\vch^*, \vcx_{T:t}) \| \leq \scL\scF^{\scL-1}\beta_\sct.
\end{equation}
\end{theorem}

The proof follows directly from \cite[Theorem 4]{Gama2020stability} for a generic $\beta_t$. In our case, $\beta_t$ is the bound corresponding to the first term on the r.h.s. of \eqref{eq:tvnn_filter_stab}, which decreases with the number of samples with rate $\mathcal{O}(1/\sqrt{t})$. The stability of STVNN decreases with the number of layers and embedding size $F$ as for the generic GNN~\cite{Gama2020stability}. We remark that quantifying the suboptimality of the parameters for the STVNN is challenging due to the non-convexity of the problem, but even so this may also not be insightful as during training we will end in local minima.

\smallskip
\noindent\textbf{Comparison with alternative bounds.} As an alternative to the STVNN, we may ignore the structural temporal dependencies and consider the VNN~\cite{sihag2022covariance} developed for tabular. This implies that we treat the $T-1$ previous values of the time series as node features. From~\cite[Theorem 3]{sihag2022covariance} this model has a stability bound $\Delta_\text{VNN} = L(FT)^{L-1}\beta_t$, where $L$ is the number of layers and $\beta_t$ is the filter stability bound that does not depend on $T$. This bound grows exponentially with $T$, whereas the bound of the STVNN grows only linearly with $T$, indicating its superior stability.

\smallskip
\noindent\textbf{Comparison with PCA bound.} To provide further insight into the STVNN, we derive the following stability analysis for the online PCA w.r.t. the uncertainties in the covariance matrix. 
\begin{proposition}
\label{pr:pca_stability}
Consider the eigendecompositions of the true covariance matrix $\mtC=\mtV\mtLambda\mtV^\Tr$ and of the sample covariance $\Cestt=\mthV\mathbf{\hat\Lambda}\mthV^\Tr$ estimated from $t$ time samples. Under As.~\ref{as_eig_diff} and given a new time sample $\vcx$, the norm difference in the projection on the true and sample covariance eigenspace is bounded as follows with probability at least $1-e^{-\epsilon}$:  
\begin{align}
\left\|\mtV^\Tr\vcx - \mthV^\Tr\vcx\right\| \leq  \frac{2N}{\sqrt{t}}\sqrt{N-1}e^{\epsilon/2}\max_{i, j\neq i}\frac{k_j}{|\lambda_i-\lambda_j|} + \mathcal{O}\left(\frac{1}{t}\right), \label{eq:pca_stab_bound}
\end{align}
where $\lambda_i$ is the $i$-th largest eigenvalue of $\mtC$ and  $k_j=\left( \mathbb{E}[||\vcx_t\vcx_t^\Tr\vcv_j||^2]-\lambda_j^2 \right)^{1/2}$ is related to the kurtosis of the data distribution.
\end{proposition}

\begin{proof}
    See Appendix~\ref{app:pca_stability}.
\end{proof}

This bound shows the lower stability of PCA compared to the STVNN and STVF. Indeed, the bound in \eqref{eq:pca_stab_bound} inversely depends on the smallest difference between two eigenvalues of the covariance matrix, which leads to instabilities when these eigenvalues are close~\cite{jolliffe1989illdefined}. For the STVF, instead, this effect is damped by the Lipschitz constant of the filter, which guarantees the output stability at the expense of discriminability. Then, the STVNN improves the discriminability of the filter by allowing to learn nonlinear mappings. In the next section, we empirically corroborate our stability analysis and show that the STVNN outperforms alternatives for forecasting tasks.

\section{Numerical Results}
\label{sec:experiments}

This section rigorously corroborates the performance of the proposed approach both quantitatively and qualitatively with synthetic and three real datasets. Our main focus is to position the STVNN w.r.t. the temporal PCA and alternatives.

\smallskip
\noindent\textbf{Datasets.} We consider two synthetic datasets and three real-world datasets to investigate the stability of STVNN and its adaptability to distribution shifts. We split datasets into train/validation/test sets of size 20\%/10\%/70\% to simulate a streaming data setting. A thorough overview of the datasets is provided in \Cref{app:datasets}.

\emph{Synthetic datasets.} We generate both stationary and non-stationary time series. In the stationary setting, we fix a covariance matrix $\mtC$ and sample observations $\vcz_t \sim \mathcal{N}(\textbf{0}, \mtC)$. Then, we create time series by enforcing temporal causality as $\vcx_t = \sum_{t'=0}^{\tau} h_{t'}\vcz_{t-t'}$, where $h_{t'} = h'_{t'}/\sqrt{\sum_{t}h'_t}$ and $h'_{t} = e^{-t}$ for $t=0,\dots,\tau$. In our experiments, we set $\tau=9$ and we generate datasets with covariances $\mtC$ with varying eigenvalue distribution tail sizes (related to kurtosis and eigenvalue closeness).
In the non-stationary setting, we consider a first-order autoregressive process $\vcx_t = \alpha_t\vcx_{t-1}+\vceps$, where $\vceps_i \sim \mathcal{N}(0,1)$. We fix $\alpha_t=0.5$ in the training set of 4000 samples and, in the test set, we modify $\alpha_t$ every 1000 samples setting it to, in order, 0.1, 0.4, 0.6, 0.1, 0.3, 0.6.
We repeat all experiments on synthetic datasets on 10 different time series generated with the same parameter configuration and report the average results.

\emph{Real-world datasets.} We consider three real datasets of different temporal sizes, resolutions and dynamics: i) the \textbf{NOAA} dataset comprising 8579 hourly temperature measurements across 109 stations in the US~\cite{arguez2012noa}; ii) the \textbf{Molene} dataset comprising 744 temperature recordings with an hourly resolution over 32 stations located in a region of France~\cite{girault2015stationary}; and iii) the \textbf{Exchange rate} dataset containing the daily exchange rates of 8 countries' currencies for 7588 open days~\cite{wu2020connecting}.
For NOAA and Molene, we apply the same preprocessing steps in~\cite{isufi2019varma}, whereas for Exchange rate we consider the dataset provided in the public repository in~\cite{MTGNNrepo}.

\smallskip
\noindent\textbf{Experimental setup.}
On synthetic datasets, we use STVNN with 2 layers of size \{32,16\} and $K=2$; on real datasets, instead, we optimize the hyperparameters through a grid search. To perform forecasting, we apply a 2-layer MLP on the embeddings generated by STVNN and use the Mean Squared Error loss.
To prevent numerical issues, since we observe the principal eigenvalue of the estimated covariance matrix to be large for some datasets, we use the trace-normalized estimated covariance matrix (i.e., $\mttC_t=\Cestt/\trace(\Cestt)$).
STVNNs are trained for 40 epochs. We run forecasting experiments on real datasets 5 times and report average performance and standard deviation. 
We provide additional details about our experimental setup and the hyperparameter grids in~\Cref{app:exp_setup}.

\begin{figure}[t]
  \centering
  \begin{subfigure}{.45\textwidth}
\includegraphics[width=\linewidth]{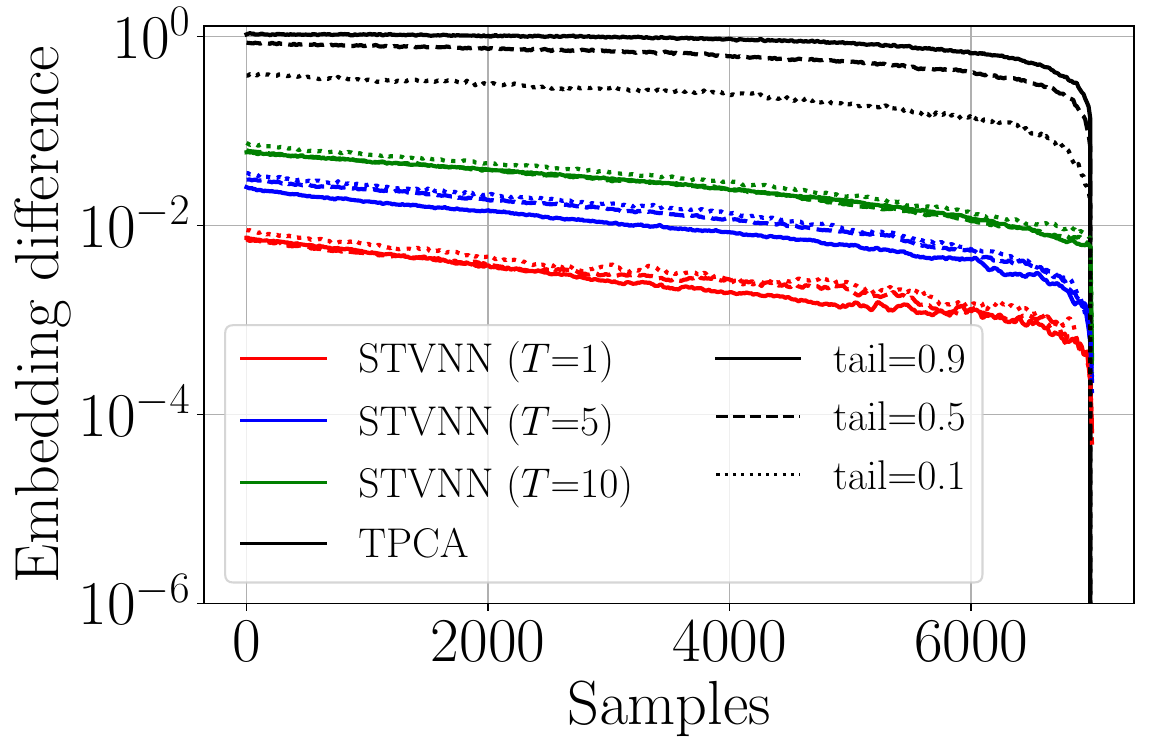}
\caption{Synthetic datasets}
\label{fig:stability_syn}
  \end{subfigure}
  \hfill
    \begin{subfigure}{.45\textwidth}        
    \includegraphics[width=\linewidth]{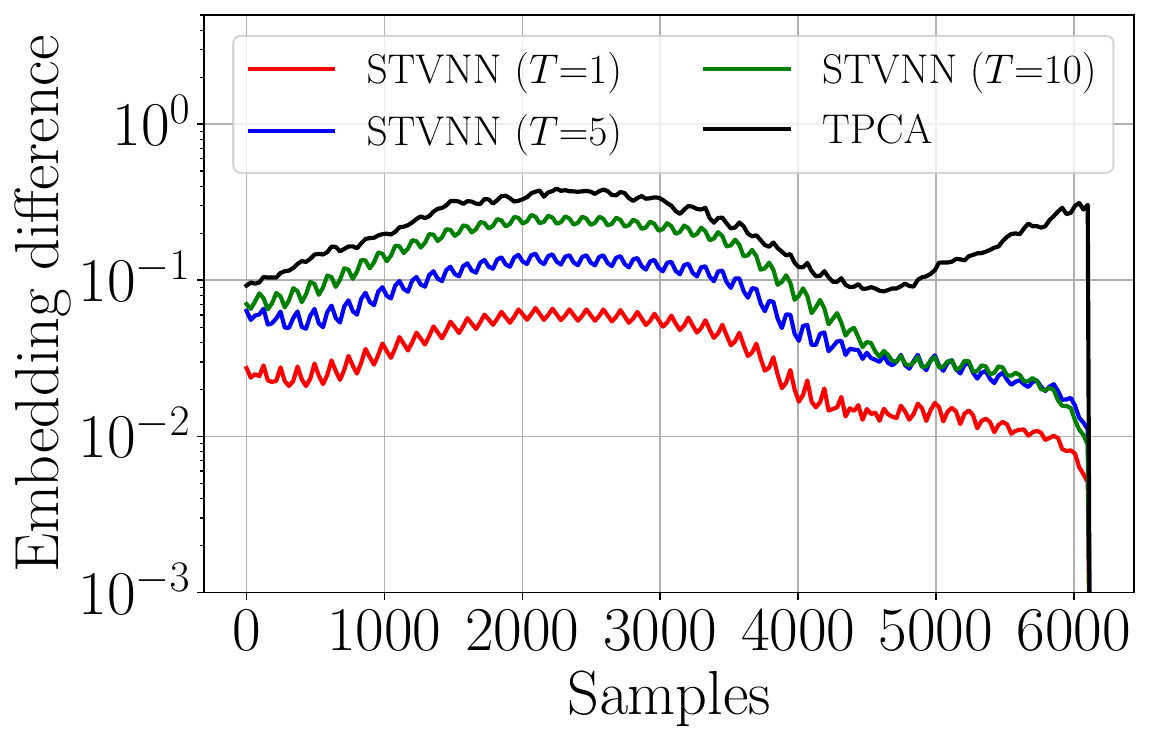}
  \caption{NOAA}
\label{fig:stability_noa}
  \end{subfigure}
  \hfill
    \begin{subfigure}{.45\textwidth}
\includegraphics[width=\linewidth]{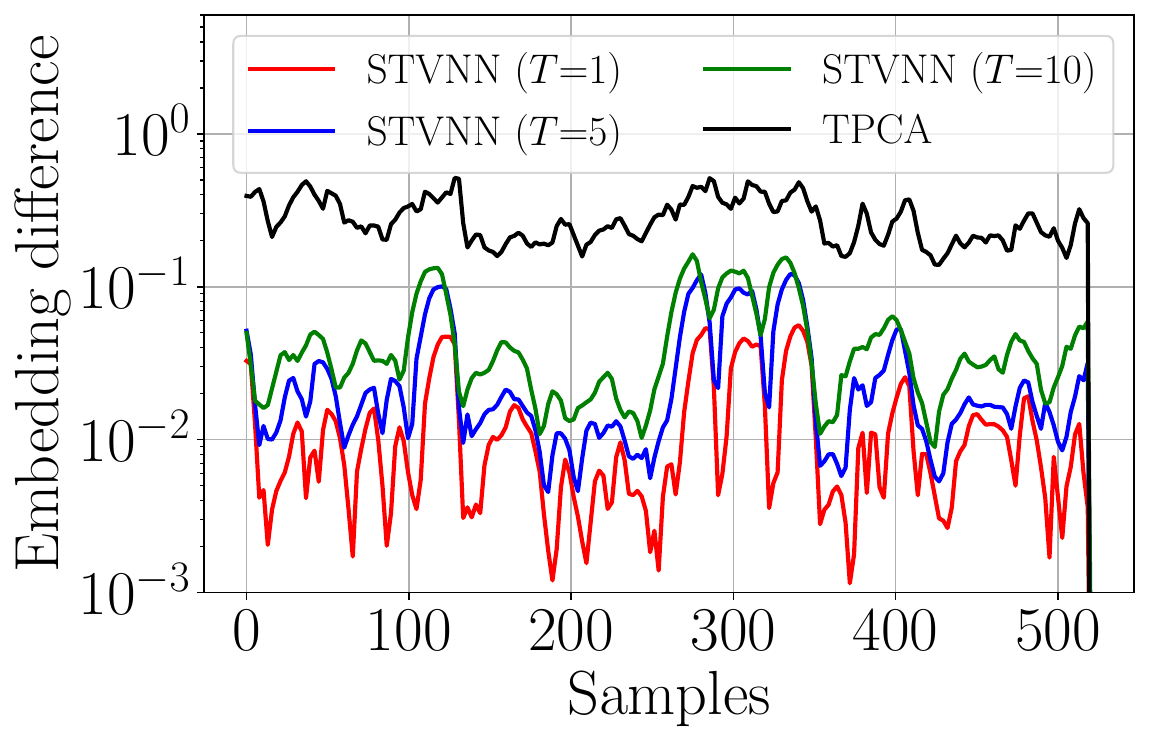}
\caption{Molene}
\label{fig:stability_molene}
  \end{subfigure}
  \hfill
    \begin{subfigure}{.45\textwidth}        
    \includegraphics[width=\linewidth]{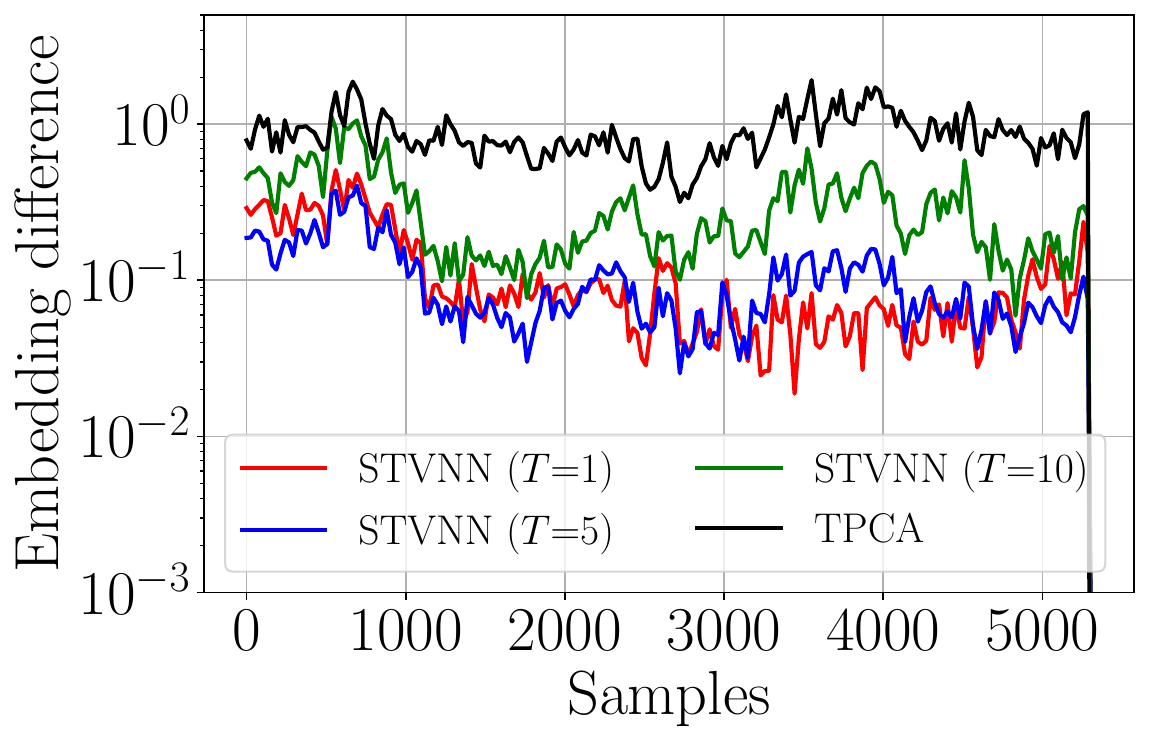}
  \caption{Exchange rate}
\label{fig:stability_exchange_rate}
  \end{subfigure}
  \caption{Embedding difference of the STVNN (i.e., $\| \fnPhi(\Cestt,\vch^*, \vcx_{T:t}) - \fnPhi(\mtC,\vch^*, \vcx_{T:t}) \|$) and TPCA with estimated and optimal covariance and parameters on different datasets for distinct observation windows $T$ and covariance eigenvalues distribution tails on synthetic datasets. Larger tails imply closer eigenvalues and higher kurtosis, leading to less distinguishable principal components.}
\label{fig:stability}

\end{figure}

\subsection{Model analysis}
\label{sec:numerical_stability}
We first analyze the behavior of STVNN in stationary and non-stationary settings by discussing its stability and capability to adapt to distribution shifts and we assess the impact of the covariance and parameters updates.

\begin{figure}[t]
\centering
\begin{minipage}{.49\textwidth}
  \centering
  \includegraphics[width=\linewidth]{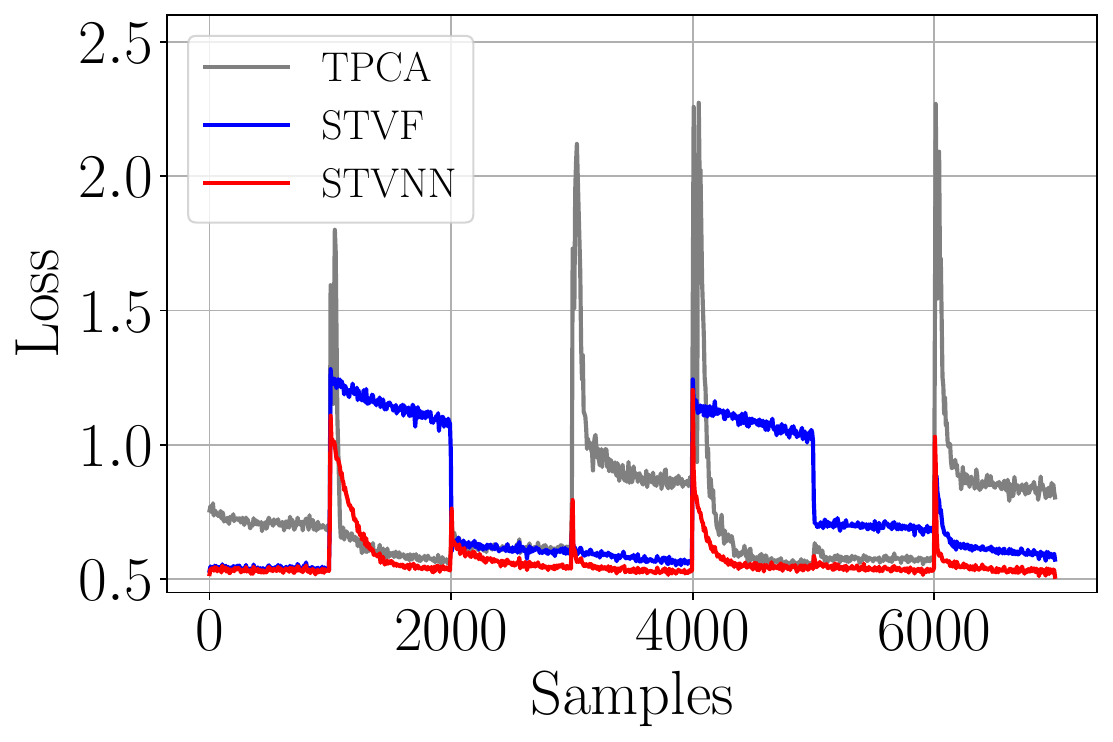}
  \caption{Loss evolution on the synthetic non-stationary dataset.}
  \label{fig:distr_shifts}
\end{minipage}%
\hfill
\begin{minipage}{.49\textwidth}
  \centering
  \includegraphics[width=\linewidth]{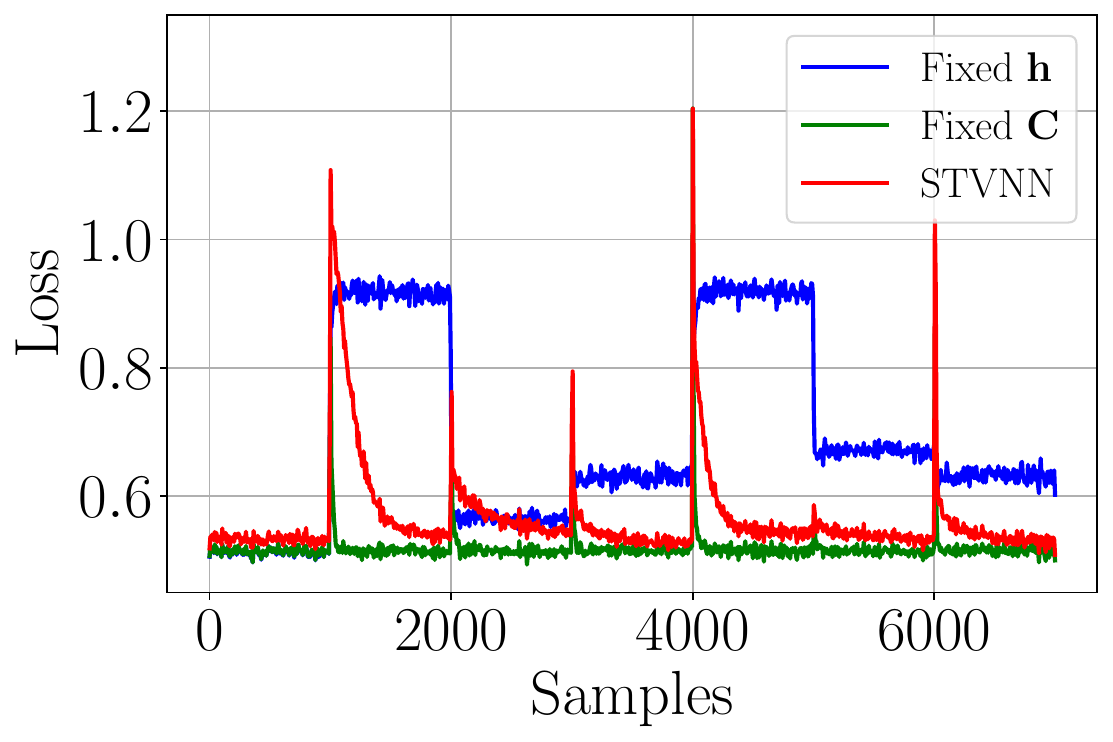}
  \captionof{figure}{Ablation study on non-stationary datasets.}
  \label{fig:ablations}
\end{minipage}
\end{figure}

\smallskip
\noindent\textbf{Stability.} In Fig.~\ref{fig:stability}, we corroborate the stability of the STVNN when trained online from finite samples and contrast it with the temporal PCA (TPCA). We consistently see that the STVNN outperforms TPCA for any temporal window size $T$, confirming the superior stability of graph convolutions over PCA-based approaches. This is accentuated in the synthetic case with a high distribution tail (Fig.~\ref{fig:stability_syn})  where the eigenvalues are closer to each other. STVNN is little affected by the distribution tail, whereas TPCA becomes significantly less stable when the eigenvalues are closer. This aligns with our theoretical observations in \eqref{eq:tvnn_stability} and \eqref{eq:pca_stab_bound}. Finally, we observe that as $T$ gets larger, the STVNN pays in stability, ultimately, corroborating result \eqref{eq:tvnn_filter_stab}.

\smallskip
\noindent\textbf{Distribution shifts.} To evaluate the capability of STVNN to adapt to distribution shifts, we evaluate the loss on the test set of the synthetic non-stationary datasets. Fig.~\ref{fig:distr_shifts} compares the STVNN, the spatiotemporal covariance filter, and the TPCA. The STVNN adapts quickly to the distribution shifts, whereas a single filter converges more slowly. Since the filter is a linear function, there exists only one optimal point, which may take long to reach, whereas the non-convexity of STVNN leads to multiple local optima that can be reached more quickly. We also see that the online TPCA leads to a substantially worse performance, especially in the proximity of the point changes.

\smallskip
\noindent\textbf{Online updates.} Finally, we investigate the role of the different online updates in the STVNN. Specifically, we train the STVNN by fixing alternatively either the parameters $\vch$ or the covariance matrix during the online updates on the test set. In Fig.~\ref{fig:ablations}, we report the results for the synthetic non-stationary datasets. The model with fixed parameters cannot adapt to distribution shifts, thus its performance drops significantly when the drifts happen. Differently, when the model is allowed to adapt the parameters with the fixed true covariance matrix (i.e., computed using the complete dataset), it adapts more quickly. By contrasting this (fixed $\mtC$) with the STVNN performance we can appreciate the impact of the covariance matrix estimation error and its role in slowing down the adaptation process. However, we remark that once the drift is passed the STVNN quickly approaches the optimal case.

\begin{table}[t]
\caption{Symmetric Mean Absolute Percentage Error (sMAPE) for forecasting on real datasets. Best results are in \textbf{bold}, second best are in \textit{\underline{italic}}.}
\label{tab:forecasting}
\centering
\setlength{\tabcolsep}{4pt}
\begin{tabular}{c|c|ccccc}
\toprule
\textbf{Datasets} & \textbf{Steps} & \textbf{LSTM} & \textbf{TPCA} & \textbf{VNN} & \textbf{VNNL} & \textbf{STVNN} \\
\midrule
\multirow{3}{*}{NOAA} & 1 & 1.98$\pm$0.07 & 2.42$\pm$0.05 & 1.71$\pm$0.04 & \textit{\underline{1.67}}$\pm$0.10 & \textbf{1.35}$\pm$0.06 \\
 & 3 & \textbf{3.10}$\pm$0.10 & 3.93$\pm$0.06 & \textit{\underline{3.14}}$\pm$0.10 & 3.36$\pm$0.27 & 3.21$\pm$0.07 \\
 & 5 & \textbf{3.46}$\pm$0.18 & 5.10$\pm$0.09 & 4.03$\pm$0.34 & 5.76$\pm$2.37 & \textit{\underline{3.71}}$\pm$0.25 \\
 \midrule
 \multirow{3}{*}{Molene} & 1 & 0.29$\pm$0.00 & 0.35$\pm$0.00 & \textbf{0.20}$\pm$0.01 & 0.21$\pm$0.00 & \textbf{0.20}$\pm$0.01 \\
 & 3 & 0.47$\pm$0.01 & 0.46$\pm$0.01 & 0.43$\pm$0.01 & \textit{\underline{0.41}}$\pm$0.01 & \textbf{0.38}$\pm$0.03 \\
 & 5 & 0.60$\pm$0.01 & \textbf{0.56}$\pm$0.00 & 0.64$\pm$0.04 & 0.59$\pm$0.00 & \textbf{0.56}$\pm$0.02 \\
\midrule
\multirow{3}{*}{\shortstack{Exchange\\rate}} & 1 & 1.25$\pm$0.02 & 1.73$\pm$0.18 & 0.70$\pm$0.01 & \textit{\underline{0.68}}$\pm$0.01 & \textbf{0.65}$\pm$0.01 \\
 & 3 & 1.33$\pm$0.03 & 1.66$\pm$0.02 & \textit{\underline{0.98}}$\pm$0.02 & 1.00$\pm$0.01 & \textbf{0.94}$\pm$0.01 \\
 & 5 & 1.39$\pm$0.02 & 1.75$\pm$0.04 & 1.19$\pm$0.01 & \textit{\underline{1.16}}$\pm$0.02 & \textbf{1.11}$\pm$0.01 \\
\bottomrule
\end{tabular}
\end{table}

\subsection{Forecasting}

Furthermore, we investigate the forecasting potential of the STVNN on three real datasets, so as to reach the following two objectives:
\begin{itemize}
    \item[]\textbf{(O1)} show the importance of the covariance networks as inductive biases for spatiotemporal relational learning;
    \item[] \textbf{(O2)} highlight the role of joint spatiotemporal learning w.r.t. disjoint models or models that ignore the temporal dependencies.
\end{itemize}
To reach \textbf{(O1)}, we contrast the STVNN with a vector LSTM model that is trained online on the multivariate time series. Instead, to reach \textbf{(O2)}, we contrast it with the disjoint models: i) TPCA, which first processes the data with the temporal PCA to extract features and then uses an MLP for the forecasting task; ii) VNNL (VNN-LSTM), a static VNN followed by an LSTM shared across the different time series. For objective \textbf{(O2)}, we also contrast with the vanilla VNN from~\cite{sihag2022covariance} where the previous $T$ time snapshots are treated as node features.

\smallskip
\Cref{tab:forecasting} reports the performance of all methods on multi-step forecasting. The STVNN achieves the most consistent performance among the baselines across datasets and steps, demonstrating its competitiveness with popular models such as LSTM. This justifies the choice of using the empirical covariance matrix as an inductive bias to model data interdependencies. We also see a consistent improvement over the VNN with time features and VNN-LSTM, justifying the importance of the temporal convolution in the STVF compared to alternative ways of incorporating temporal information. Lastly, we see a substantial improvement w.r.t. TPCA-based regression showing that our two-dimensional causal filter in \eqref{eq:tv_filter} compensates well the lack of the cross-covariance matrices and at the same time keeps the computational costs contained. We show that the same behavior occurs also for different forecasting performance metrics in \Cref{app:mts_forecasting_full}.

\section{Conclusion}

This work introduced the spatiotemporal covariance neural network (STVNN), a graph neural network that leverages spatiotemporal convolutions using the covariance matrix as an inductive bias to represent their interactions. STVNN is suitable for streaming and non-stationary data as it can adapt both the covariance network and parameters online.
We proved that STVNN is stable to uncertainties due to these online estimations and its temporal component moderately affects this stability compared to alternative designs such as temporal principal component analysis.
We corroborated our theoretical findings with numerical experiments showing that STVNN is stable in streaming settings, adapts to distribution shifts, and is effective in forecasting tasks. 
Compared to other covariance-based temporal data processing techniques such as temporal PCA, STVNN does not account for covariance terms at time lags different from zero, modeling instead interdependencies across time with a temporal sum. This facilitates scalability while capturing the spatiotemporal relations. Yet, developing relational learning models that include also these terms may provide other insights into STVNN. Future work will focus on this aspect as well as theoretically characterizing the impact of distribution shifts.

\begin{credits}
\subsubsection{\ackname} 
The study was supported by the TU Delft AI Labs program and the NWO OTP GraSPA proposal \#19497.

\subsubsection{\discintname}
The authors have no competing interests to declare that are relevant to the content of this article. 
\end{credits}

%
%
\bibliographystyle{splncs04}
\bibliography{bibliography}

\newpage
\include{supplementary}

\end{document}

%% file: supplementary.tex
\appendix

\section{Temporal PCA}
\label{app:tpca}
We consider the common temporal PCA (TPCA) approach in~\cite[Chapter 12]{Jolliffe2002pca} in a stationary setting. 
At time $t$, we consider the extended data sample $\vctx_t = [\vcx_t^\Tr,\vcx_{t-1}^\Tr,\dots,\vcx_{t-T+1}^\Tr]^\Tr$ given by the concatenation of the current and the previous $T-1$ time samples. Then, we compute the extended sample covariance matrix 
\begin{equation}
\mttC = \mathbb{E} [(\vctx_t-\boldsymbol{\tilde{\mu}})(\vctx_t-\boldsymbol{\tilde{\mu}})^\Tr],    
\end{equation}
where $\boldsymbol{\tilde{\mu}}$ is the sample mean. Temporal PCA consists in the projection of a data sample $\vctx_t$ on the eigenspace of $\mttC$.  
More in detail, the matrix $\mttC$ captures correlations through space and time given its block structure: 
\begin{equation}
\mttC = 
\begin{bmatrix}
\mtC \quad & \mtC_{1} \quad & \dots \quad & \mtC_{T-1} \\
\mtC_{1} \quad & \mtC \quad & \dots \quad & \mtC_{T-2} \\
\vdots \quad & \vdots \quad & \ddots \quad & \vdots \quad & \\
\mtC_{T-1} \quad & \mtC_{T-2} \quad & \dots \quad & \mtC \\
\end{bmatrix}	    
\end{equation}
where the matrixes on the diagonal are lag-zero covariance matrixes and the off-diagonal terms are the covariances at different time lags as defined in \Cref{remark1}.
Throughout the paper, when we perform a downstream task using TPCA-based baselines, we first project the extended data samples $\vctx_t$ on the covariance matrix eigenspace and then apply a readout layer (generally an MLP) for the final task. This, however, has a prohibitive computational cost of order $\mathcal{O}(N^3T^3)$ for the eigendecomposition and of quadratic order for every projection of $\vctx_t$.

\section{Proof of \Cref{th:stability_filter}}
\label{app:proof_stability_filter}

To bound the filter output difference $\|\mtH(\Cestt,\vch_t,\vcx_{T:t}) - \mtH(\mtC,\vch^*,\vcx_{T:t})\|$, we add and subtract within the norm $\mtH(\Cestt,\vch^*,\vcx_{T:t})$ to get
\begin{align}
\begin{split}
        &\left\|\mtH(\Cestt,\vch_t,\vcx_{T:t}) - \mtH(\mtC,\vch^*,\vcx_{T:t})\right\| = \\
        &\left\|\mtH(\Cestt,\vch_t,\vcx_{T:t}) - \mtH(\mtC,\vch^*,\vcx_{T:t}) + \mtH(\Cestt, \vch^*,\vcx_{T:t}) - \mtH(\Cestt, \vch^*,\vcx_{T:t})\right\| \leq \\
        &\left\|\mtH(\Cestt,\vch_t,\vcx_{T:t}) - \mtH(\Cestt, \vch^*,\vcx_{T:t})\right\| + \left\|\mtH(\Cestt, \vch^*,\vcx_{T:t}) - \mtH(\mtC,\vch^*,\vcx_{T:t})\right\| \leq \\
        & \alpha_t + \beta_t,
\end{split}
\end{align}
where the penultimate step follows from the triangle inequality. Here, we see the contributions of the two terms in the bound: $\|\mtH(\Cestt,\vch_t,\vcx_{T:t}) - \mtH(\Cestt, \vch^*,\vcx_{T:t})\| \leq \alpha_t$ (parameters suboptimality error) and $\|\mtH(\Cestt, \vch^*,\vcx_{T:t}) - \mtH(\mtC,\vch^*,\vcx_{T:t})\| \leq \beta_t$ (covariance uncertainty error). In the sequel, we upper-bound each of them. We shall specify that all the norms used in the proof are $\ell_2-$norms if the argument is a vector and operator norms if it is a matrix.

\subsection{Parameters suboptimality error}

Given $\vch^*$ as the optimal set of coefficients for the multivariate time series forecasting problem, optimizing $\vch_t$ is equivalent to solving the problem 
\begin{equation}
    \min_{\vch_t}\left\|\mtH(\Cestt,\vch_t, \vcx_{T:t}) - \mtH(\Cestt, \vch^*, \vcx_{T:t})\right\|^2.
\end{equation}
This is an unconstrained convex optimization problem since the cost function is the composition of a convex norm and an affine transformation of the coefficients $\vch$ (i.e., the filter). Therefore, since we update $\vch_t$ via online gradient descent, under the assumption of a learning rate $\eta>0$ small enough to guarantee convergence, we can exploit the upper-bound of the distance between the online update and the optimal solution in~\cite[Theorem 3.4]{garrigos2023handbook} to obtain $\alpha_t \le \|\vch^*\|^2/({2 \eta t})$. 

\subsection{Covariance uncertainty error}

To begin with, we consider that the estimated covariance matrix $\Cestt$ is a perturbed version of the true covariance matrix $\mtC$ such that $\Cestt = \mtC + \mtE$, where $\mtE$ is an error matrix. Then, we write the Taylor expansion for $\Cestt^\sck$ as
\begin{equation}
    \Cestt^\sck = (\mtC + \mtE) ^\sck = \mtC^\sck + \sum_{r=0}^{k-1}\mtC^r\mtE\mtC^{k-r-1} + \mttE
\end{equation}
where $\mttE$ is such that $\|\mttE\| = \mathcal{O}(\|\mtE\|^2)$. Since we consider the error to be small $\|\mtE\| \ll 1$, which is the case after a few training steps of the STVNN (i.e., $t \gg 0$), we ignore $\mttE$ in the following.

Then, we define the eigendecomposition of the true covariance matrix $\mtC = \mtV\mtLambda\mtV^\Tr$ and of the estimated one $\Cestt = \mthV\mathbf{\hat{\Lambda}}\mthV^\Tr$. Substituting the eigendecomposition of $\mtC$ and applying the respective graph Fourier transform on signal $\vcx_{t-t'}$, we get

\begin{align}
\begin{split}
\mtH(\Cestt, \vch_t, \vcx_{T:t}) - \mtH(\mtC, \vch^*, \vcx_{T:t}) = \sum_{t'=0}^{T-1}\sum_{k=0}^K h_{kt'} \sum_{r=0}^{k-1}\mtC^r\mtE\mtC^{k-r-1}\vcx_{t-t'} \\
    = \sum_{i=0}^{N-1}\sum_{t'=0}^{T-1}\sctx_{t-t',i} \sum_{k=0}^K h_{kt'} \sum_{r=0}^{k-1}\mtC^r\mtE\mtC^{k-r-1}\vcv_i  \\
    = \sum_{i=0}^{N-1}\sum_{t'=0}^{T-1}\sctx_{t-t',i} \sum_{k=0}^K h_{kt'} \sum_{r=0}^{k-1}\mtC^r\sclambda_i^{k-r-1}\mtE\vcv_i \label{eq:evi}
\end{split}
\end{align}
where $\sctx_{t-t',i}$ is the $i$-th entry of $\vctx_{t-t^\prime} = \mtV^{\Tr}\vcx_{t-t^\prime}$. We now expand the last term as
\begin{equation}\label{eq_split}
    \mtE\vcv_i = \mtB_i\delta\vcv_i + \delta\lambda_i\vcv_i + (\delta\lambda_i\mtI_N-\mtE)\delta\vcv_i
\end{equation}
where
\begin{equation*}
    \mtB_i = \lambda_i\mtI_N-\mtC, \quad \delta\vcv_i = \vchv_i - \vcv_i, \quad \delta\lambda_i = \hat\lambda_i - \lambda_i.
\end{equation*}
Then, upon substituting \eqref{eq_split} into \eqref{eq:evi}, we get
\begin{align}
    \sum_{i=0}^{N-1}\sum_{t'=0}^{T-1}\sctx_{t-t',i} \sum_{k=0}^K h_{kt'} \sum_{r=0}^{k-1}\mtC^r\sclambda_i^{k-r-1}\mtB_i\delta\vcv_i  \label{eq:t1} \\
    +\sum_{i=0}^{N-1}\sum_{t'=0}^{T-1}\sctx_{t-t',i} \sum_{k=0}^K h_{kt'} \sum_{r=0}^{k-1}\mtC^r\sclambda_i^{k-r-1}\delta\lambda_i\vcv_i  \label{eq:t2} \\
    +\sum_{i=0}^{N-1}\sum_{t'=0}^{T-1}\sctx_{t-t',i} \sum_{k=0}^K h_{kt'} \sum_{r=0}^{k-1}\mtC^r\sclambda_i^{k-r-1}(\delta\lambda_i\mtI_N-\mtE)\delta\vcv_i. \label{eq:t3}
\end{align}
In the remaining part of the proof, we proceed with upper-bounding each term individually.

\subsubsection{First term \eqref{eq:t1}.}
We note that $\mtB_i = \lambda_i\mtI_N-\mtC = \mtV(\lambda_i\mtI_N-\mtLambda)\mtV^\Tr$. Plugging this into \eqref{eq:t1}, we get
\begin{align}\label{eq.t1expand}
\begin{split}
    \sum_{i=0}^{N-1}\sum_{t'=0}^{T-1}\sctx_{t-t',i} \sum_{k=0}^K h_{kt'} \sum_{r=0}^{k-1}\mtC^r\sclambda_i^{k-r-1}\mtV(\lambda_i\mtI_N-\mtLambda)\mtV^\Tr\delta\vcv_i  \\
    = \sum_{i=0}^{N-1}\sum_{t'=0}^{T-1}\sctx_{t-t',i} \sum_{k=0}^K h_{kt'} \sum_{r=0}^{k-1}\sclambda_i^{k-r-1}\mtV\mtLambda^r(\lambda_i\mtI_N-\mtLambda)\mtV^\Tr\delta\vcv_i \\
    = \sum_{i=0}^{N-1}\sum_{t'=0}^{T-1}\sctx_{t-t',i} \mtV\mtL_{it'}\mtV^\Tr(\vchv_i-\vcv_i)
\end{split}
\end{align}
where $\mtL_{it'}$ is a diagonal matrix whose $j$-th diagonal element is 0 if $i=j$ and for if $i\neq j$ it is
\begin{align}
\begin{split}
    \mtL_{it',j} &= \sum_{k=0}^K h_{kt'} \sum_{r=0}^{k-1}\sclambda_i^{k-r-1}\sclambda_j^r(\lambda_i-\lambda_j)
    =  \sum_{k=0}^K h_{kt'} \frac{\sclambda_i^k-\sclambda_j^k}{\lambda_i-\lambda_j}(\lambda_i-\lambda_j) \\
    &= \sum_{k=0}^K h_{kt'}\sclambda_i^k - \sum_{k=0}^K h_{kt'}\sclambda_j^k = h_{t'}(\lambda_i) - h_{t'}(\lambda_j).
\end{split}
\end{align}
Here, $h_{t'}(\lambda)$ is the graph frequency response of filter $\mtH_{\sct'}(\mtC):=\sum_{\sck=0}^\scK \sch_{\sck\sct'}\mtC^\sck$.
Therefore, 
\begin{align}
    [\mtL_{it'}\mtV^\Tr(\vchv_i-\vcv_i)]_j=\begin{cases}
			0, & \text{if $i=j$}\\
                (h_{t'}(\lambda_i) - h_{t'}(\lambda_j))\vcv_j^\Tr\vchv_i, & \text{if $i \neq j$}
		 \end{cases}
\end{align}
Taking the norm of \eqref{eq.t1expand}, we get 
\begin{align}
\begin{split}
    \left\|\sum_{i=0}^{N-1}\sum_{t'=0}^{T-1}\sctx_{t-t',i} \mtV\mtL_{it'}\mtV^\Tr(\vchv_i-\vcv_i)\right\| \le  \sum_{i=0}^{N-1}\sum_{t'=0}^{T-1}|\sctx_{t-t',i}| \|\mtV\|~\|\mtL_{it'}\mtV^\Tr(\vchv_i-\vcv_i)\| \leq \\
    \sqrt{N}\sum_{i=0}^{N-1}\sum_{t'=0}^{T-1}|\sctx_{t-t',i}|\max_j|h_{t'}(\lambda_i) - h_{t'}(\lambda_j)||\vcv_j^\Tr\vchv_i| 
\end{split}
\end{align}
where the first inequality derives from the triangle and Cauchy-Schwarz inequalities and the second one holds from $\|\mtV\| = 1$ alongside $\|\vcy\|\leq\sqrt{N}\max_{i=1}y_i$ for an arbitrary vector $\vcy \in \real^N$.
We now leverage the result from \cite[Theorem~4.1]{loukas2017howclose} to characterize the dot product of the eigenvectors of true and sample covariance matrixes under As.~\ref{as_eig_diff} as
\begin{align}\label{eq.event1}
    \mathbb{P}(|\vcv_j^\Tr\vchv_i|\geq B) \leq \frac{1}{t}\left(\frac{2k_j}{B|\lambda_i-\lambda_j|}\right)^2
\end{align}
where $k_j=\left( \mathbb{E}[||\vcx_t\vcx_t^\Tr\vcv_j||^2_2]-\lambda_j^2 \right)^{1/2}$ is related to the kurtosis of the data distribution~\cite{loukas2017howclose,sihag2022covariance}.
By setting 
\begin{align}
    B = \frac{2k_je^{\epsilon/2}}{t^{1/2}|\lambda_i-\lambda_j|},
\end{align}
we get
\begin{align}
    \max_j|h_{t'}(\lambda_i) - h_{t'}(\lambda_j)||\vcv_j^\Tr\vchv_i| \leq \max_j\frac{|h_{t'}(\lambda_i) - h_{t'}(\lambda_j)|}{|\lambda_i-\lambda_j|}\frac{2k_je^{\epsilon/2}}{t^{1/2}}
\end{align}
with probability at least $1-e^{-\epsilon}$.

As per As.~\ref{as_lipschitz}, all filters' frequency responses $h_{t'}(\lambda)$ are Lipschitz with constant $P$. 
Therefore, the term in \eqref{eq:t1} is bounded as
\begin{align}
\begin{split}
\label{eq:first_term}
    \left\|\sum_{i=0}^{N-1}\sum_{t'=0}^{T-1}\sctx_{t-t',i} \sum_{k=0}^K h_{kt'} \sum_{r=0}^{k-1}\mtC^r\sclambda_i^{k-r-1}\beta_i\delta\vcv_i\right\| \leq \\
    \sum_{i=0}^{N-1}\sum_{t'=0}^{T-1}|\sctx_{t-t',i}| \frac{2P\sqrt{N} k_{\text{max}}e^{\epsilon/2}}{t^{1/2}} \leq \frac{2}{t^{1/2}}Pk_{\text{max}}e^{\epsilon/2}TN
\end{split}
\end{align}
with probability at least $1-e^{-\epsilon}$. Note that we leveraged $\|\vcx_{t}\|\leq1$ for all $t$, $\sum_{i=0}^{N-1} |\tilde{x}_{t,i}| \leq \sqrt{N}\|\vcx_{t}\|$, and defined $k_{\text{max}} := \max_jk_j$.

\medskip
\noindent\textbf{Second term \eqref{eq:t2}.}
We rewrite \eqref{eq:t2} as
\begin{align}
\begin{split}
    \sum_{i=0}^{N-1}\sum_{t'=0}^{T-1}\sctx_{t-t',i} \sum_{k=0}^K h_{kt'} \sum_{r=0}^{k-1}\mtC^r\sclambda_i^{k-r-1}\delta\lambda_i\vcv_i& = \sum_{i=0}^{N-1}\sum_{t'=0}^{T-1}\sctx_{t-t',i} \sum_{k=0}^K h_{kt'} \sum_{r=0}^{k-1}\sclambda_i^r\sclambda_i^{k-r-1}\delta\lambda_i\vcv_i \\
    &= \sum_{i=0}^{N-1}\sum_{t'=0}^{T-1}\sctx_{t-t',i} \sum_{k=0}^K kh_{kt'}\sclambda_i^{k-1}\delta\lambda_i\vcv_i \\
    &= \sum_{i=0}^{N-1}\sum_{t'=0}^{T-1}\sctx_{t-t',i} h'_{t'}(\lambda_i)\delta\lambda_i\vcv_i.
\end{split}
\end{align}
where $h'_{t'}(\lambda)$ is the derivative of $h_{t'}(\lambda)$ w.r.t. $\lambda$.
Taking the norm and applying standard inequalities, we get 
\begin{equation}
    \left\|\sum_{i=0}^{N-1}\sum_{t'=0}^{T-1}\sctx_{t-t',i} h'_{t'}(\lambda_i)\delta\lambda_i\vcv_i\right\| \le \sum_{i=0}^{N-1}\sum_{t'=0}^{T-1}|\sctx_{t-t',i}|~ |h'_{t'}(\lambda_i)|~|\delta\lambda_i|~\|\vcv_i\|.
\end{equation}
Here, we have that $\|\vcv_i\| = 1$, and that derivative of the filter $h'_{t'}(\lambda)$ is bounded by $P$ from As.~\ref{as_lipschitz}. We now proceed with bounding $|\delta\lambda_i|$. From Weyl's theorem~\cite[Theorem 8.1.6]{golub13}, we note that $\|\mtE\|\leq \alpha$ implies that $|\delta\lambda_i|\leq \alpha$ for any $\alpha > 0$.
Next, using the result from~\cite[Theorem 5.6.1]{vershynin2018high}, we have 
\begin{align}\label{eq.event2}
    \mathbb{P}\bigg(\|\mtE\| \leq \underbrace{Q\left( \sqrt{\frac{G^2N(\log N+u)}{t}}+\frac{G^2N(\log N+u)}{t} \right)\|\mtC\|}_{\alpha} \bigg) \geq 1-2e^{-u}.
\end{align}
where $Q$ is an absolute constant and $G\geq 1$ derives from As.~\ref{as_norm}. Finally, $|\tilde{x}_{t-t^\prime,i}|$ is handled via inequality $\sum_{i=0}^{N-1} |\tilde{x}_t| \leq \sqrt{N}\|\vcx_t\|$ and leveraging $\|\vcx_t\|\leq 1$. Putting all these together, we upper-bound the second term \eqref{eq:t2} by
\begin{align}
\label{eq:second_term}
\begin{split}
    \left\|\sum_{i=0}^{N-1}\sum_{t'=0}^{T-1}\sctx_{t-t',i} \sum_{k=0}^K h_{kt'} \sum_{r=0}^{k-1}\mtC^r\sclambda_i^{k-r-1}\delta\lambda_i\vcv_i\right\| \leq \\
    PT\sqrt{N}Q\left( \sqrt{\frac{G^2N(\log N+u)}{t}}+\frac{G^2N(\log N+u)}{t} \right)\|\mtC\|
\end{split}
\end{align}
with probability at least $1-2e^{-u}$.

\smallskip
\noindent\textbf{Third term \eqref{eq:t3}.} This term can be bounded by leveraging equations (65)-(68) in \cite{sihag2022covariance} with minimal changes and showing that $\|(\delta\lambda_i\mtI_N-\mtE)\delta\vcv_i\|$ scales as $\mathcal{O}(1/t)$.

Bringing together \eqref{eq:first_term}, \eqref{eq:second_term} and the observation that \eqref{eq:t3} scales as $\mathcal{O}(1/t)$ leads to the bound in \eqref{eq:tvnn_filter_stab} given events \eqref{eq.event1} and \eqref{eq.event2} are independent. \qed

\begin{remark}
    Note that events \eqref{eq:first_term} and \eqref{eq:second_term} are generally dependent as both rely on the underlying and sample covariance matrices. However, they are coupled in a nontrivial way that brings in eigenspace alignments and thus it is challenging to quantify the exact probability for which the bound in \eqref{eq:tvnn_filter_stab} holds. If these events negatively affect each other the probability that the bound in \eqref{eq:tvnn_filter_stab} holds will be smaller than $(1-e^{-\epsilon})(1-2e^{-u})$, but if the events positively affect each other the bound in \eqref{eq:tvnn_filter_stab} may hold with a higher probability. These considerations, however, do not involve the actors that appear in the bound, whose contribution to the model stability remains an insightful analysis.
\end{remark}

\section{Proof of \Cref{pr:pca_stability}}
\label{app:pca_stability}
We consider w.l.o.g. $\|\vcx\|\leq 1$. We have that 
\begin{align}
\begin{split}
\label{eq:initialeqbound}
    \left\|\mtV^\Tr\vcx - \mthV^\Tr\vcx\right\| &= \left\|\sum_{i=0}^{N-1} (\vcv_i-\vchv_i)^\Tr x_i\right\| \leq \sum_{i=0}^{N-1} \|\vcv_i-\vchv_i\||x_i| \\
    &\le \sum_{i=0}^{N-1} \|\vcv_i-\vchv_i\| \leq N \max_i\|\vcv_i-\vchv_i\|
\end{split}
\end{align}
where we used triangle inequality and the fact that $\|(\vcv_i-\vchv_i)^\Tr\| = \|\vcv_i-\vchv_i\|$.
Now, we split the difference $\delta\vcv_i = \vcv_i-\vchv_i$ into a component parallel to $\vcv_i$ and another perpendicular to $\vcv_i$: $\delta\vcv_i = \delta\vcv_{i\parallel} + \delta\vcv_{i\perp}$.
The eigenvectors of a covariance matrix (i.e., both $\mtV$ and $\mthV$) are an orthonormal basis of $\mathbb{R}^N$; thus, the perpendicular component of the error $\delta\vcv_{i\perp}$ is the sum of the projections of $\vchv_i$ on all the columns of $\mtV$ except $\vcv_i$, i.e., 
\begin{align}
    \delta\vcv_{i\perp} = \sum_{j=0,j\neq i}^{N-1} (\vchv_i^\Tr\vcv_j)\vcv_j.
\end{align}
If we take the norm, we get
\begin{align}
\|\delta\vcv_{i\perp}\| = \sqrt{\sum_{j=0,j\neq i}^{N-1} (\vchv_i^\Tr\vcv_j)^2} \leq \sqrt{N-1} \max_{j, j\neq i}|\vchv_i^\Tr\vcv_j|.
\end{align}
The norm of the parallel component is, instead, 
\begin{align}
\label{eq:parallel}
    \|\delta\vcv_{i\parallel}\| = \big|~\|\vcv_i\|- \vchv_i^\Tr\vcv_i~\big|.
\end{align}
Since eigenvectors are invariant to change of sign (i.e., if $\vcv_i$ is an eigenvector, $-\vcv_i$ is equivalently an eigenvector), we assume w.l.o.g. $\vchv_i^\Tr\vcv_i\geq0$.
Therefore, we have 
\begin{equation}
    \|\delta\vcv_{i\parallel}\| = \big|~\|\vcv_i\|- \vchv_i^\Tr\vcv_i~\big| = \|\vcv_i\|- |\vchv_i^\Tr\vcv_i|.
\end{equation}

Since $\|\mthV\|$ and $\|\mtV\|$ are orthonormal bases, we have that $\|\vcv_i\| = \|\vchv_i\| = 1$ and we can express $\vchv_i$ in terms of its projections on the directions in $\mtV$, i.e., $\vchv_i = \sum_{j=0}^{N-1} (\vchv_i^\Tr\vcv_j)\vcv_j$. 
Therefore, if we take the norm, we have that $\|\vchv_i\| = \sqrt{\sum_{j=0}^{N-1} (\vchv_i^\Tr\vcv_j)^2} = 1$.  We can use this expression to write 
\begin{equation}
\label{eq:squareroot}
    |\vchv_i^\Tr\vcv_i| = \sqrt{1 - \sum_{j=0,j\neq i}^{N-1} (\vchv_i^\Tr\vcv_j)^2} = 1 - \frac{1}{2}\sum_{j=0,j\neq i}^{N-1} (\vchv_i^\Tr\vcv_j)^2 - \mathcal{O}\left( \left(\sum_{j=0,j\neq i}^{N-1}(\vchv_i^\Tr\vcv_j)^2\right)^2\right)
\end{equation}
where we used the Taylor expansion of the square root as its argument approaches 1, i.e., the eigenvector approximations get closer to the true eigenvectors.

Plugging \eqref{eq:squareroot} in \eqref{eq:parallel} and using the fact that $\|\vcv_i\|$ = 1, we get

\begin{align}
    \|\delta\vcv_{i\parallel}\| \leq \|\vcv_i\| - 1 + \frac{1}{2}\sum_{j=0,j\neq i}^{N-1} (\vchv_i^\Tr\vcv_j)^2 + \mathcal{O}\left( \left(\sum_{j=0,j\neq i}^{N-1}(\vchv_i^\Tr\vcv_j)^2\right)^2\right) = \\ 
    \frac{1}{2}\sum_{j=0,j\neq i}^{N-1} (\vchv_i^\Tr\vcv_j)^2 + \mathcal{O}\left( \left(\sum_{j=0,j\neq i}^{N-1}(\vchv_i^\Tr\vcv_j)^2\right)^2\right) \leq \\
    \frac{1}{2}(N-1) \max_{j, j\neq i}|\vchv_i^\Tr\vcv_j|^2 + \mathcal{O}\left( \max_{j, j\neq i}|\vchv_i^\Tr\vcv_j|^4 \right).
\end{align}

Now, we can rewrite the norm of $\delta\vcv_i$ as
\begin{align}
\begin{split}
\label{eq:differencebound}
    \|\delta\vcv_i\| \leq \|\delta\vcv_{i\perp}\| + \|\delta\vcv_{i\parallel}\| \leq \\\
    \sqrt{N-1} \max_{j, j\neq i}|\vchv_i^\Tr\vcv_j| + \frac{1}{2}(N-1) \max_{j, j\neq i}|\vchv_i^\Tr\vcv_j|^2 + \mathcal{O}\left( \max_{j, j\neq i}|\vchv_i^\Tr\vcv_j|^4 \right).
    \end{split}
\end{align}

We now use the result from \cite[Theorem~4.1]{loukas2017howclose} that holds under As.~\ref{as_eig_diff}:
\begin{align}
    \mathbb{P}(|\vcv_j^\Tr\vchv_i|\geq B) \leq \frac{1}{t}\left(\frac{2k_j}{B|\lambda_i-\lambda_j|}\right)^2
\end{align}
where $k_j=\left( \mathbb{E}[||\vcx_t\vcx_t^\Tr\vcv_j||^2_2]-\lambda_j^2 \right)^{1/2}$ is related to the kurtosis of the data distribution~\cite{loukas2017howclose,sihag2022covariance}. If we set 
\begin{align}
    B = \frac{2k_je^{\epsilon/2}}{t^{1/2}|\lambda_i-\lambda_j|},
\end{align}
we get 
\begin{align}
    \label{eq:dotproductbound}
    \mathbb{P}\left(|\vcv_j^\Tr\vchv_i|\leq \frac{1}{\sqrt{t}}\frac{2k_je^{\epsilon/2}}{|\lambda_i-\lambda_j|}\right) \geq 1-e^{-\epsilon}.
\end{align}
Therefore, by plugging the results of \eqref{eq:dotproductbound} and \eqref{eq:differencebound} into \eqref{eq:initialeqbound}, we obtain that the following bound holds with probability at least $1-e^{-\epsilon}$:

\begin{align}
||\mtV^\Tr\vcx - \mtU^\Tr\vcx|| \leq  \frac{2N}{\sqrt{t}}\sqrt{N-1}e^{\epsilon/2}\max_{i,j\neq i}\frac{k_j}{|\lambda_i-\lambda_j|} + \mathcal{O}\left(\frac{1}{t}\right),
\end{align}
where $\mathcal{O}(1/t)$ collects the term related to $|\vchv_i^\Tr\vcv_j|^2$ and $|\vchv_i^\Tr\vcv_j|^4$ in \eqref{eq:differencebound}. 
\qed

\section{Synthetic Dataset Generation}
\label{app:datasets}
\subsection{Stationary datasets}
Given a covariance matrix $\mtC$, we generate stationary datasets by first sampling observations $\vcz_t \sim \mathcal{N}(\textbf{0}, \mtC)$ and then enforcing temporal causality by creating the multivariate time series
temporal observations $\vcx_t$ as $\vcx_t = \sum_{t'=0}^{\tau} h_{t'}\vcz_{t-t'}$, where $h_{t'} = h'_{t'}/\sqrt{\sum_{t}h'_t}$ and $h'_{t} = e^{-t}$ for $t=0,\dots,\tau$. In our experiments we set $\tau=9$.
To control the covariance eigenvalues distribution in the synthetic dataset, we generate different covariances through the function \lstinline{sklearn.make_regression}, that allows to set the size of the distribution tail and, consequently, control how close the eigenvalues are. \Cref{fig:eig_distr} shows the eigenvalue distribution for the three covariance matrixes that we use to generate our synthetic datasets. Bigger tail sizes correspond to higher kurtosis and, ultimately, closer eigenvalues.

\begin{figure}[h]
     \centering
     \includegraphics[width=0.5\textwidth]{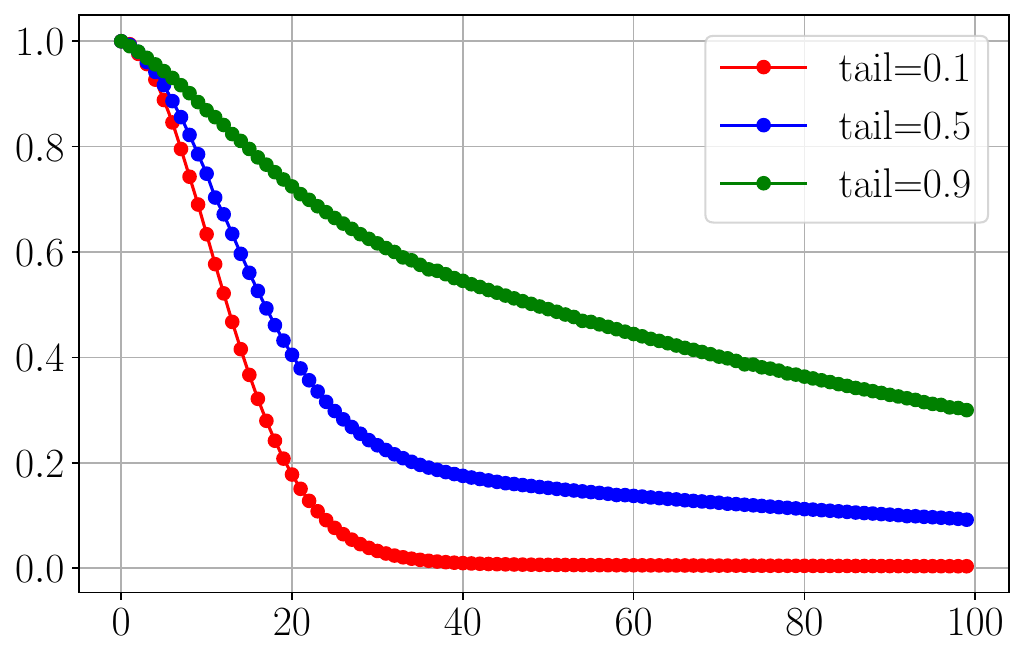}
     \caption{Eigenvalue distribution for different tail sizes. }
     \label{fig:eig_distr}
\end{figure}

\subsection{Non-stationary datasets}
For non-stationary datasets, instead, we use a first-order autoregressive process $\vcx_t = \alpha_t\vcx_{t-1}+\vceps$, where $\vceps$ is a random vector whose $i$-th component is $\vceps_i \sim \mathcal{N}(0,1)$ and we vary $\alpha_t$ every 1000 samples. We generate the first sample as $\vcx_0 \sim \mathcal{N}(\textbf{0}, \mtC)$, where $\mtC$ is generated using \lstinline{sklearn.make_regression} with tail strength 0.1. We generate 4000 samples as training set using $\alpha_t=0.5$ and, in the test set, we change the values of $\alpha_t$ every 1000 samples in the following order: 0.1, 0.4, 0.6, 0.1, 0.3, 0.6.

\section{Experimental setup}
\label{app:exp_setup}

We optimize the following hyperparameters for STVNN through a search.
\begin{itemize}
    \item Number of layers: \{1,2,3\}
    \item Feature size per-layer: \{8,16,32,64,128\}
\item Learning rate: \{0.001, 0.0001\}
\item Optimizer: \{Adam, SGD\}
\item Order of graph filters $K$: \{1,2,3\}
\item $\gamma$: \{0.01,0.1,0.3\}
\item $T$: \{2,3,5,8\}
\end{itemize}
The non-linearity in STVNN is LeakyReLU with negative slope 0.1. 
To prevent numerical issues, since we observe the principal eigenvalue of the estimated covariance matrix to be large for some datasets, we employ the trace-normalized estimated covariance matrix (i.e., $\mttC_t=\Cestt/\trace(\Cestt)$) as graph shift operator in our architectures.
We train models for 40 epochs. We split datasets into train/validation/test sets of size 20\%/10\%/70\% to simulate a streaming data setting.
For VNN and VNN-LSTM, we use the best hyperparameter configuration of STVNN per dataset. For TPCA and LSTM, we only optimize $T$.
We write our code in Python and we use PyTorch for deep learning models and optimization.

For experiments on stability, we use the best parameter configuration on real datasets and, on synthetic datasets, we use 2 layers of size \{32,16\} and $K=2$. TPCA has $T=2$ on all stability experiments. We use Adam for training and SGD for online updates.
For experiments on non-stationary synthetic datasets, we use the same configuration as for stability and we set $\gamma=0.1$. We use Adam for training and SGD for online updates.
For all experiments on synthetic benchmarks we average results over 10 different datasets generated with the same parameters. 

\section{Forecasting results}
\label{app:mts_forecasting_full}

\Cref{tab:forecasting_full} reports a more comprehensive view of the results for multivariate time series forecasting experiments in \Cref{tab:forecasting} by showing also the Mean Squared Error (MSE) and Mean Absolute Error (MAE) in addition to the symmetric Mean Absolute Percentage Error (sMAPE). Overall, we see a similar trend as reported in \Cref{tab:forecasting}.

\begin{table}[t]
\caption{MTS forecasting results on real datasets. Best results are in \textbf{bold}, second best are in \textit{\underline{italic}}. On Exchange Rate, MSE is $\times10^{-4}$ and MAE is $\times10^{-2}$.}
\label{tab:forecasting_full}
\centering
\resizebox{\linewidth}{!}{
\begin{tabular}{c|c|ccc|ccc|ccc}
\toprule
 & Steps & \multicolumn{3}{|c|}{1} & \multicolumn{3}{|c|}{3} & \multicolumn{3}{|c}{5}  \\
 \midrule
 & & MSE & MAE & sMAPE (\%) & MSE & MAE & sMAPE (\%) & MSE & MAE & sMAPE (\%) \\
\midrule
\multirow{5}{*}{\rotatebox[origin=c]{90}{\textbf{NOAA}}} & LSTM & 2.94$\pm$0.12 & 1.17$\pm$0.05 & 1.98$\pm$0.07 & 7.04$\pm$0.46 & \textit{\underline{1.85}}$\pm$0.06 & \textbf{3.10}$\pm$0.10 & 9.03$\pm$0.81 & \textbf{2.02}$\pm$0.10 & \textbf{3.46}$\pm$0.18 \\
& TPCA & 3.71$\pm$0.33 & 1.34$\pm$0.04 & 2.42$\pm$0.05 & 8.84$\pm$0.30 & 2.18$\pm$0.04 & 3.93$\pm$0.06 & 14.9$\pm$0.44 & 2.84$\pm$0.06 & 5.10$\pm$0.09 \\
& VNN & 1.98$\pm$0.09 & 1.00$\pm$0.02 & 1.71$\pm$0.04 & \textbf{5.64}$\pm$0.53 & \textbf{1.80}$\pm$0.07 & \textit{\underline{3.14}}$\pm$0.10 & \textit{\underline{8.91}}$\pm$1.20 & 2.28$\pm$0.17 & 4.03$\pm$0.34 \\
& VNNL & \textit{\underline{1.76}}$\pm$0.16 & \textit{\underline{0.97}}$\pm$0.05 & \textit{\underline{1.67}}$\pm$0.10 & 6.40$\pm$0.84 & 1.98$\pm$0.15 & 3.36$\pm$0.27 & 20.3$\pm$15.3 & 3.45$\pm$1.51 & 5.76$\pm$2.37 \\
& STVNN & \textbf{1.22}$\pm$0.14 & \textbf{0.79}$\pm$0.04 & \textbf{1.35}$\pm$0.06 & \textit{\underline{6.34}}$\pm$0.58 & \textit{\underline{1.85}}$\pm$0.06 & 3.21$\pm$0.07 & \textbf{8.37}$\pm$0.14 & \textit{\underline{2.13}}$\pm$0.16 & \textit{\underline{3.71}}$\pm$0.25 \\
\midrule
\multirow{5}{*}{\rotatebox[origin=c]{90}{\textbf{Molene}}} & LSTM & 1.27$\pm$0.01 & 0.80$\pm$0.00 & 0.29$\pm$0.00 & 3.35$\pm$0.06 & 1.30$\pm$0.02 & 0.47$\pm$0.01 & 5.09$\pm$0.12 & 1.67$\pm$0.01 & 0.60$\pm$0.01 \\
& TPCA & 1.73$\pm$0.03 & 0.97$\pm$0.01 & 0.35$\pm$0.00 & 2.99$\pm$0.10 & 1.29$\pm$0.02 & 0.46$\pm$0.01 & \textit{\underline{4.20}}$\pm$0.03 & \textbf{1.57}$\pm$0.01 & \textbf{0.56}$\pm$0.00 \\
& VNN & \textit{\underline{0.57}}$\pm$0.02 & \textbf{0.56}$\pm$0.02 & \textbf{0.20}$\pm$0.01 & 2.47$\pm$0.10 & 1.19$\pm$0.03 & 0.43$\pm$0.01 & 5.58$\pm$0.78 & 1.81$\pm$0.10 & 0.64$\pm$0.04 \\
& VNNL & 0.62$\pm$0.02 & 0.58$\pm$0.01 & 0.21$\pm$0.00 & \textit{\underline{2.28}}$\pm$0.06 & \textit{\underline{1.14}}$\pm$0.01 & \textit{\underline{0.41}}$\pm$0.01 & 4.49$\pm$0.02 & 1.67$\pm$0.00 & 0.59$\pm$0.00 \\
& STVNN & \textbf{0.58}$\pm$0.04 & \textbf{0.56}$\pm$0.02 & \textbf{0.20}$\pm$0.01 & \textbf{2.03}$\pm$0.03 & \textbf{1.06}$\pm$0.01 & \textbf{0.38}$\pm$0.03 & \textbf{4.19}$\pm$0.23 & \textbf{1.57}$\pm$0.05 & \textbf{0.56}$\pm$0.02 \\
\midrule
\multirow{5}{*}{\rotatebox[origin=c]{90}{\textbf{Exchange}}} & LSTM & 2.10$\pm$0.07 & 0.88$\pm$0.01 & 1.25$\pm$0.02 & 2.51$\pm$0.17 & 0.95$\pm$0.02 & 1.33$\pm$0.03 & 2.75$\pm$0.06 & 1.00$\pm$0.01 & 1.39$\pm$0.02 \\
& TPCA & 4.03$\pm$0.93 & 1.16$\pm$0.11 & 1.73$\pm$0.18 & 3.58$\pm$0.10 & 1.14$\pm$0.01 & 1.66$\pm$0.02 & 4.07$\pm$0.01 & 1.21$\pm$0.02 & 1.75$\pm$0.04 \\
& VNN & 0.84$\pm$0.03 & 0.52$\pm$0.01 & 0.70$\pm$0.01 & 1.48$\pm$0.06 & 0.72$\pm$0.01 & \textit{\underline{0.98}}$\pm$0.02 & 2.04$\pm$0.01 & 0.86$\pm$0.01 & 1.19$\pm$0.01 \\
& VNNL & \textit{\underline{0.71}}$\pm$0.02 & \textit{\underline{0.49}}$\pm$0.00 & \textit{\underline{0.68}}$\pm$0.01 & \textit{\underline{1.32}}$\pm$0.03 & \textit{\underline{0.70}}$\pm$0.01 & 1.00$\pm$0.01 & \textbf{1.73}$\pm$0.01 & \textbf{0.81}$\pm$0.01 & \textit{\underline{1.16}}$\pm$0.02 \\
& STVNN & \textbf{0.69}$\pm$0.02 & \textbf{0.48}$\pm$0.01 & \textbf{0.65}$\pm$0.01 & \textbf{1.31}$\pm$0.04 & \textbf{0.68}$\pm$0.01 & \textbf{0.94}$\pm$0.01 & \textit{\underline{1.81}}$\pm$0.02 & \textit{\underline{0.82}}$\pm$0.01 & \textbf{1.11}$\pm$0.01 \\
\bottomrule
\end{tabular}}
\end{table}